\documentclass[10pt,twocolumn,letterpaper]{article}

\usepackage[pagenumbers]{cvpr} % To force page numbers, e.g. for an arXiv version

\usepackage{graphicx}
\usepackage{amsmath}
\usepackage{amssymb}
\usepackage{booktabs}
\usepackage[pagebackref,breaklinks,colorlinks]{hyperref}

\usepackage[capitalize]{cleveref}
\crefname{section}{Sec.}{Secs.}
\Crefname{section}{Section}{Sections}
\Crefname{table}{Table}{Tables}
\crefname{table}{Tab.}{Tabs.}

\def\cvprPaperID{9935} % *** Enter the CVPR Paper ID here
\def\confName{CVPR}
\def\confYear{2023}

\usepackage{microtype}
\usepackage{graphicx}
\usepackage{amsmath}
\usepackage{mathtools}
\usepackage{amsthm}
\usepackage{enumerate}
\usepackage[utf8]{inputenc} % allow utf-8 input
\usepackage[T1]{fontenc}    % use 8-bit T1 fonts
\usepackage{url}            % simple URL typesetting
\usepackage{booktabs}       % professional-quality tables
\usepackage{amsfonts}       % blackboard math symbols
\usepackage{nicefrac}       % compact symbols for 1/2, etc.
\usepackage{xcolor}         % colors
\usepackage[export]{adjustbox}
\usepackage{float}
\usepackage{bbm}

\definecolor{darkgreen}{RGB}{204,102,0}

\newcommand{\rev}[1]{{\color{purple}{#1}}}

\makeatletter
\newcommand{\oset}[3][0ex]{%
  \mathrel{\mathop{#3}\limits^{
    \vbox to#1{\kern-2\ex@
    \hbox{$\scriptstyle#2$}\vss}}}}
\makeatother
\newcommand{\optimal}[1]{\oset{\scalebox{.5}{$\star$}}{#1}}

\newcommand{\cA}{\mathcal{A}}

 \newcommand{\cL}{\mathcal{L}}
 \newcommand{\cN}{\mathcal{N}}

 \newcommand{\cY}{\mathcal{Y}}
 \newcommand{\cX}{\mathcal{X}}

\newcommand{\EE}{\mathbb{E}} \newcommand{\RR}{\mathbb{R}}

\newcommand{\optclf}{\optimal{\theta}}
\newcommand{\fairclf}{\theta_{\mathrm{f}}}
\newcommand{\robclf}{\theta^{(\epsilon)}_{\mathrm{r}}}
\newcommand{\naterr}[1]{\cL^{\mathrm{nat}}_{#1}}
\newcommand{\bdyerr}[1]{\cL^{\mathrm{bdy}}_{#1}\left(\epsilon\right)}
\newcommand{\roberr}[1]{\cL^{\mathrm{rob}}_{#1}\left(\epsilon\right)}

\newcommand{\pnDist}{\mu_+-\mu_-}
\newcommand{\Kbound}{B_K}
\newcommand{\newKbound}{\bar{B}_K}

\newcommand{\dist}[2]{\Delta\left(#1,f_{#2}\right)}
\newcommand{\expect}[1]{\EE\left[#1\right]}

\newtheorem{theorem}{Theorem}[section]
\newtheorem{proposition}[theorem]{Proposition}

\newtheorem{corollary}[theorem]{Corollary}
\theoremstyle{definition}

\theoremstyle{remark}

\DeclareMathOperator*{\argmin}{argmin}

\newcommand{\pr}[1]{\operatorname{Pr}\left(#1\right)}

\begin{document}
\title{Fairness Increases Adversarial Vulnerability}

\author{
Cuong Tran\\
Syracuse University\\
{\tt\small cutran@syr.edy}
\and
Keyu Zhu\\
Georgia Tech\\
{\tt\small kzhu67@gatech.edu}
\and
Ferdinando Fioretto\\
Syracuse University\\
{\tt\small ffiorett@syr.edy}
\and
Pascal Van Hentenryck\\
Georgia Tech\\
{\tt\small pvh@isye.gatech.edu}
}

\maketitle

\begin{abstract}
The remarkable performance of deep learning models and their applications in consequential 
domains (e.g., facial recognition) introduces important challenges at the 
intersection of equity and security. Fairness and robustness are two desired notions
often required in learning models. Fairness ensures that models do not disproportionately 
harm (or benefit) some groups over others, while robustness measures the models' 
resilience against small input perturbations. 

This paper shows the existence of a dichotomy between fairness and robustness, and analyzes 
when achieving fairness decreases the model robustness to adversarial samples. 
The reported analysis sheds light on the factors causing such contrasting behavior, 
suggesting that distance to the decision boundary across groups as a key explainer for 
this behavior. Extensive experiments on non-linear models and 
different architectures validate the theoretical findings in multiple vision domains. 
Finally, the paper proposes a simple, yet effective, solution to construct models 
achieving good tradeoffs between fairness and robustness.
\end{abstract}

\section{Introduction}

Data-driven learning systems have become instrumental for decision-making in a variety of consequential 
contests. They include assistance in legal decisions  \cite{jayatilake2021involvement}, lending  \cite{stevens2020explainability}, 
hiring  \cite{schumann2020we}, performing personalized ads targeting  \cite{choi2020identifying}, and providing personalized recommendations \cite{burke2003hybrid}.
As a result, fairness has become a crucial requirement for their successful use and adoption. Various notions of 
fairness drawing from legal and philosophical doctrine have been proposed to ensure that the models errors do not disproportionately affect the decisions of some groups over others \cite{mehrabi2021survey}. In general, fair models attempt at constraining their hypothesis space so that the errors of the reported outcomes are 
distributed uniformly across different protected groups. 

When these fairness constraints are enforced in learning systems, a commonly observed behavior is an overall
degradation of the model accuracy. Thus, a growing body of research has been focusing on striking the right balance between fairness and accuracy \cite{rodolfa2021empirical}. 
This paper shows that fairness may have another important consequence on the deployed models: \emph{a reduction of the model robustness}. This aspect is important as the vulnerability of deep learning models to adversarial 
examples hinders their application in many security-sensitive domains. However, these behaviors are currently
not fully understood and have not received the attention they deserve given the significant equity and security 
consequences they have on the final decisions. 

This paper addresses this important gap and shows that enforcing fairness may negatively affect the robustness of a model. In particular, the paper makes the following  contributions: (1) it analyzes when and why fairness and robustness may be misaligned in their objectives, (2) it provides an analysis on the relationship between fair, robust, and "natural" (e.g., non-fair non-robust) models, and (3) it identifies \emph{the distance to the decision boundary} as a key aspect linking fairness and robustness. Moreover, (4) the paper shows how the distance to the decision boundary  can explain the increase of adversarial vulnerability of fair models, providing extensive experiments and validation over a variety of vision tasks and architectures, and verifying the presence of the fairness/robustness dichotomy for multiple techniques aimed at achieving fairness and measuring robustness. Finally, (5) building from the reported theoretical observations, the paper also proposes a simple, yet effective, strategy to find a good tradeoff between accuracy, fairness, and robustness. 

To the best of the authors' knowledge, this is is the first work showing that enforcing fairness may negatively affect robustness. The results show that, without careful considerations, inducing a desired equity property on a learning task may create significant security challenges. These results should not be read as an endorsement to avoid constructing fairer or safer models; rather it should be understood as a call for additional research at the intersection of fairness and robustness to achieve appropriate tradeoffs.

\section{Related work}
\label{sec:related_work}

\noindent\textbf{Fairness.}
Models that learns from rich datasets have been shown to carry over 
bias which may induce to disproportionally harm some groups of individuals (often identified by their race or gender) over others \cite{buolamwini2018gender,cuong_neurips22,krishnan2020understanding}.
These observations have resulted in a whole new research area that has focused on defining, analyzing, and mitigating unfairness \cite{friedler2019comparative,xu2021consistent}.
The source of the observed unfairness has been often attributed to data properties \cite{subramanian2021fairness,caton2020fairness,mehrabi2021survey} or different aspects of the model's properties \cite{sukthanker2022importance,NEURIPS2021_e7e8f8e5}. For example, imbalance in groups' size is commonly argued to create disparities in the task's performance \cite{mehrabi2021survey}. It has also been shown that 
% whenmodel's training is performed to enforce
constraining the model's hypothesis space to satisfy privacy \cite{bagdasaryan2019differential}, sparsity \cite{hooker2019selective,hooker2020characterising}, or robustness \cite{xu2021robust,nanda2021fairness} can result in disparate outcomes. 

\noindent\textbf{Robustness.}
Deep Neural Networks (DNNs) have been shown to be susceptible to carefully crafted adversarial perturbations which—imperceptible to a human—result in a misclassification by the model \cite{szegedy2013intriguing}.
The literature on the topic attributes the reason for such behavior to arise to three key factors: data properties, network architecture, and model training. 
For example, \cite{goodfellow2014explaining,shamir2019simple}  observed that the input dimension plays a decisive role in the robustness of a model, with larger inputs yielding more brittle models. Likewise, the ability of some architectures to capture high frequency spectrum of image data (that are almost imperceptible to a human) relates to the success of devising adversarial examples in the underlying models \cite{wang2020high}. %It has also been argued that the vulnerability to adversarial examples is due to the existence of predictive features that are not robust \cite{ilyas2019adversarial}. 
With respect to the network architecture, it has been observed that some network architectures may render the underlying model more or less brittle to adversarial inputs. For example, batch-normalization is cited to the models' robustness \cite{galloway2019batch}, some shift-invariant architectures, such as CNNs, are more vulnerable to adversarial inputs compared to other architectures, e.g., ~fully connected networks. %Schmidt et al.~\cite{schmidt2018adversarially}  also shows that the sample complexity needed to train a robust model may be higher than that needed to train a standard counterpart. 
Finally, Yao et al.~\cite{yao2018hessian} empirically observed that model trained with large batch-sizes can be more easily fooled by adversarial inputs when compared to models trained with smaller batch sizes. 

\noindent\textbf{Robustness and fairness.}
This work lies in the intersection of fairness and robustness. 
Within this context, 
recently Xu et al~\cite{xu2021robust} has shown that adversarially robust models exhibit remarkable disparity of natural accuracy and robust accuracy metrics among different classes, compared to those exhibited by their standard counterpart. 
Khani and Liang \cite{khani2020feature} analyze why noise in features can cause disparity in error rates when learning a regression.

We believe this is the first work to show that enforcing fairness may negatively affect a model's robustness and hope that this result will lead to further strengthening the interconnection of these two important machine learning areas.

\begin{figure*}[!t]
\centering
\hspace{16pt}\includegraphics[width=16cm]{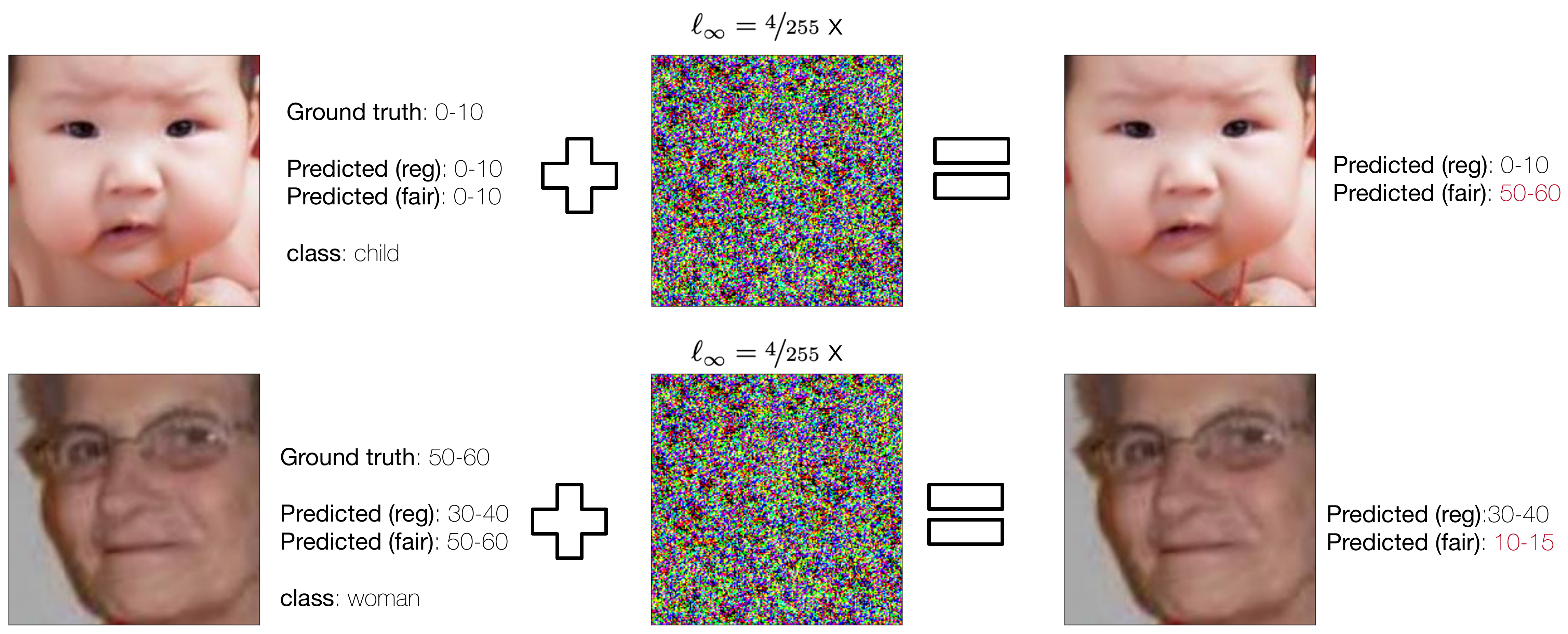}
\caption{An example of robustness loss in the UTKFace dataset. 
A regular (reg) and a fair models are trained to predict age group from faces and exposed to adversarial examples generated under an RFGSM \cite{tramer2017ensemble} attack. 
The predictions of the regular model do not change under adversarial examples (regardless of their original correctness), while the fair models decision change in the presence of adversarial noise.
\label{fig:motiv}}
\end{figure*}

%%%%%%%%%%%%%%%%%%%%%%%%%%%%%%%%%%%%%%%%%%%%%%%%%%%%%%%%%%%%%%%%%%%%%%
\section{Problem Settings}
\label{sec:settings}
%%%%%%%%%%%%%%%%%%%%%%%%%%%%%%%%%%%%%%%%%%%%%%%%%%%%%%%%%%%%%%%%%%%%%%

The paper considers a typical multi-class classification problem, 
whose input is a dataset $D$ consisting of $n$ data points $(X_i, A_i, Y_i)$, each of which drawn i.i.d.~from an unknown distribution $\Pi$ and where
$X_i \in \cX$ is a feature vector, $A_i \in \cA$ is a protected attribute, and 
$Y_i \in \cY = [C]$ is a label, with $C$ being the number of possible class labels.
For example, consider the case of a classifier to predict the age range
of an individual. The features $X_i$ may describe the pixels 
associated with the individual headshot and their demographics, the protected 
attribute $A_i$ may describe the individual gender or ethnicity, and $Y_i$ represents the age range. 
The goal is to learn a classifier $f_\theta: \cX \to \cY$, where $\theta$ is a vector of 
real-valued parameters. The model quality is assessed in terms of a non-negative loss function
$\ell: \cY \times \cY \to \mathbb{R}_+$, and the training aims at minimizing the empirical
risk function:
\begin{equation}\label{eq:erm}
    \optclf\, = \underset{\theta}{\argmin}~\cL_{\theta}(D)  
    \left( =\frac{1}{n} \sum_{i=1}^n \ell(f_{\theta}(X_i), Y_i)\right) 
\end{equation}

\noindent
For a group $a \in \cA$, notation $D_a$ is used to denote the subset of $D$ containing
exclusively samples $i$ with $A_i=a$. 
Importantly, the paper assumes that the attribute $A$ is not part of the model input during inference.
The paper focuses on learning classifiers that satisfy group fairness (to be defined shortly) and on analyzing the robustness impact of fairness. 

%%%%%%%%%%%%%%%%%%%%%%%%%%%%%%%%%%%%%%%%%%%%%%%%%%%%%%%%%%%%%%%%%%%%%%
\section{Preliminaries}
\label{sec:preliminaries}
%%%%%%%%%%%%%%%%%%%%%%%%%%%%%%%%%%%%%%%%%%%%%%%%%%%%%%%%%%%%%%%%%%%%%%
% This section reviews the fairness and robustness notions adopted in this work. 

%%%%%%%%%%%%%%%%%%%%%%%%%%%%%%%%%%%%%%%%%%%%%%%%%%%%%%%%%%%%%%%%%%%%%%
\subsection{Fairness and fair learning}
\label{sec:fair_models}
This paper considers a classifier $f$ satisfying accuracy parity \cite{zhao2019inherent}, a group fairness 
notion commonly adopted in machine learning requiring model misclassification 
rates to be conditionally independent of the protected attribute. That is, 
$\forall  (X,Y,A) \sim \Pi$ and $ \forall a \in \cA$,
\begin{equation}
\label{eq:acc_parity1}
    \left| 
    \Pr\left( f_\theta(X) \neq Y \;\vert\; A = a \right) - 
    \Pr \left( f_\theta(X) \neq Y \right) \right| \leq \alpha,
\end{equation}
\noindent
where $\alpha$ denotes the allowed \emph{fairness violation}. In practice, the above 
is expressed as a difference of empirical expectations of the group and population 
misclassification rates. That is, $\forall a \in \cA$:
\begin{equation*}
%\label{eq:acc_parity2}
    \left|
    \nicefrac{1}{|D_a|} \!\!\!\!\!\!\!\!\! \sum_{(X,A,Y) \in D_a} \!\!\!\!\!\!\!  \mathbbm{1}\{f_\theta(X) \neq Y\}  - 
    \nicefrac{1}{n} \!\! \!\!\!\!\!\!\! \sum_{(X,A,Y) \in D} \!\!\!\!\!\!\!  \mathbbm{1}\{f_\theta(X) \neq Y\}  
    \right| \leq \alpha.
\end{equation*}

% \begin{multline*}
%     \xi_{\theta}(D) = \\\max_{c \in [C]}~\left\vert \mathbb{E}\left[ \mathbbm{1}\{ f_{\theta}(X) = Y\} \mid Y=C\right] - 
%     \mathbb{E}\left[ \mathbbm{1}\{ f_{\theta}(X) = Y\}\right]\right\vert
% \end{multline*}

\noindent
Several approaches have been proposed in the literature to encourage the satisfaction of accuracy 
parity. They can be summarized in methods that use penalty terms into the empirical risk loss function 
to capture the fairness violations, and those which minimize the maximum group loss. 
{The core of the paper focuses on the first set of methods; the analysis for the second set is presented in Appendix \ref{app:fairness_defs}}.
% (1) Fair-Reg which adds a fairness constraint into the classification loss and (2) q-learn which minimizes the maximum class loss. 

\noindent\textbf{Penalty-based methods.}
In this category, the model loss function (Equation  \eqref{eq:erm}) is augmented with 
penalty fairness constraint terms \cite{agarwal2018reductions,cuong_aaai21} as follows:
\begin{align}
    \fairclf(\lambda) &=\underset{\theta}{\argmin}~\cL_{\theta}(D) + 
                          \lambda \left( \sum_{a \in \cA } \!\!\vert\cL_{\theta}(D_a) -  \cL_{\theta}(D) \vert
                          \right)
    \label{eqn:fair_reg}
\end{align}
where $\cL_{\theta}(D_a) = \nicefrac{1}{|D_a|} \sum_{(X,A,Y) \in D_a} \ell(f_{\theta}(X), Y)$  is the 
empirical risk loss associated with protected group $a \in \cA$. In addition, $\lambda > 0$ is the fairness penalty parameter that enforces a tradeoff between fairness and accuracy.

\iffalse
\noindent\textbf{Group-loss focused methods.} 
\textcolor{red}{We should considering to move this to Appendix.}

\textcolor{blue}{If you move this to the appendix, then there are sentences in the Experiments that you need to remove. I prefer not to remove it (assuming you have experiments on this) as it showcase generality, but I see that the theoretical analysis is all done on the penalty-based method. }

Methods in the second category, force the training to focus on the loss component of worst performing groups. An effective method to achieve this goal was proposed in \cite{tian_li}:
\begin{align}
    \fairclf &= \underset{\theta}{\argmin}~ \sum_{a \in \cA}\frac{1}{q+1}\cL_{\theta}(D_a)^{q+1},
    \label{eqn:q-learn}
\end{align}
where $q$ is a non-negative constant. Obviously, the intuition behind powering the loss by positive number $q+1$ is to penalize more the classes that have the larger losses. Thus, $q$ plays the role of the fairness parameter, like $\lambda$ in penalty-based methods: larger $q$ or $\lambda$ values are associated with fairer (but also often less accurate) models. 
\fi

%%%%%%%%%%%%%%%%%%%%%%%%%%%%%%%%%%%%%%%%%%%%%%%%%%%%%%%%%%%%%%%%%%%%%%%%%%%%%%%%%%%
\subsection{Robustness and Robust Learning}
\label{subsec:rob_train}

This paper analyzes the effect of enforcing fairness on adversarial robustness, a key property of trustworthy machine-learning systems. In this work, and following robust learning conventions, the robustness of a model $f$ is measured in terms of the \emph{robust error}:
\begin{equation}
\label{eq:rob_error}
    \roberr{\theta} = \pr{\exists \tau, ~\| \tau\|_p \leq \epsilon, f_{\theta}(X +\tau  ) \neq Y},
\end{equation}
which measures the sensitivity of the model errors to small input perturbations 
$\| \tau\|_p \leq \epsilon$ in $\ell_p$ norms, with $p$ often considered in $\{0, 1, 2, \infty\}$.
The robust error can be decomposed into two components \cite{zhang2019theoretically}:
\begin{equation}
\label{eq:robloss}
\roberr{\theta} =   \naterr{\theta}  + \bdyerr{\theta},
\end{equation}
where the first denotes the \emph{natural error} and the second the \emph{boundary error}.
The natural error measures the standard model performance when exposed to \emph{unperturbed} samples $(X, A, Y)$:
\begin{equation}
\label{eq:nat_error}
    \naterr{\theta} = \pr{ f_{\theta}(X) \neq Y},
\end{equation}
whose empirical version is defined in Equation \eqref{eq:erm} {with a 0/1 loss function}.
The boundary error measures the probability that the model predictions change on \emph{perturbed} samples ($X+\|\tau\|_p, A, Y)$:
\begin{align}
    \bdyerr{\theta} =
    \operatorname{Pr}& \big(\exists \|\tau \|_p \leq \epsilon, ~f_{\theta}(X +\tau) \neq f_{\theta}
    (X), \notag\\
    &f_{\theta}(X)=Y \big).
\end{align}
\noindent
The boundary error implicitly introduces a notion of {\em decision boundary} and a {\em distance between an input sample and this decision boundary}. For instance, in linear classifiers, the decision boundary is represented by an hyperplane. The distance of a sample $X$ to the decision boundary for a classifier $f_\theta$ can be formalized as 
\[
\Delta(X,f_{\theta}) = \max \epsilon \ \mbox{s.t.~} f_{\theta}(X+\tau) = f_{\theta}(X) ,\ \forall \| \tau\| \leq \epsilon.
\]
Samples close to the decision boundary will be less tolerant to noise than those lying far from it. The analysis in this paper regarding the impact of fairness on robustness is based on this concept. In particular, the results  show that imposing fairness constraints may reduce the distance to the decision boundary of the samples $(X,A,Y) \sim \Pi$.

%%%%%%%%%%%%%%%%%%%%%%%%%%%%%%%%%%%%%%%%%%%%%%%%%%%%%%%%%%%%
\section{Real-World Implications }
%%%%%%%%%%%%%%%%%%%%%%%%%%%%%%%%%%%%%%%%%%%%%%%%%%%%%%%%%%%%
Prior diving into the analysis, the paper provides an example showing how robustness errors can be exacerbated when a image classifier is trained to satisfy fairness. Deep neural networks have been used in many real-world applications, including image facial recognition and object detection. When perturbations (either due to noise or by malicious adversaries) are introduced in the model inputs, they may cause harmful effects as they lead the classifier to misclassify targeted inputs. 

Figure \ref{fig:motiv} shows an example of inputs from the UTKFace  dataset where a classifier is trained to minimize the regular empirical risk loss of equation \eqref{eq:erm} (top) or the fair empirical risk loss of equation \eqref{eqn:fair_reg} (bottom). Both inputs are perturbed with the same amount of $\ell_\infty$ noise, but the fair network is much more brittle than its regular counterpart, inducing errors in the classifier outputs. It is important to note that this paper uses datasets such as UTKFace (described in detail in Appendix \ref{app:datasets}) only to demonstrate the effects of fairness to robustness. As noted in previous works, the very task of predicting gender, race, or other characteristics from a person face is flawed and raises deep ethical concerns \cite{raji2020saving}.

 \begin{figure}[t]
 \centering
\includegraphics[width=0.7\linewidth,]{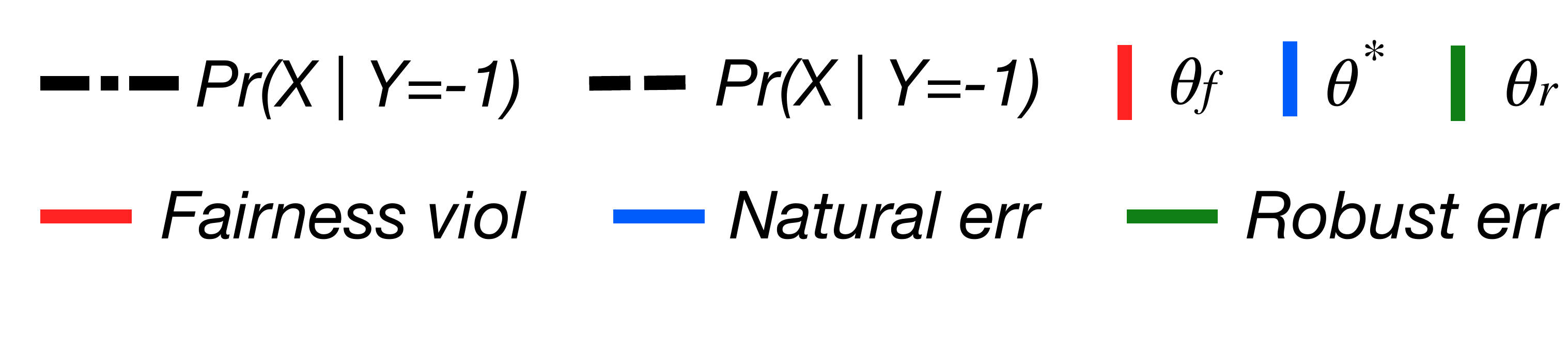}\\[-10pt]
\includegraphics[width=0.48\linewidth,height=130pt,valign=M]{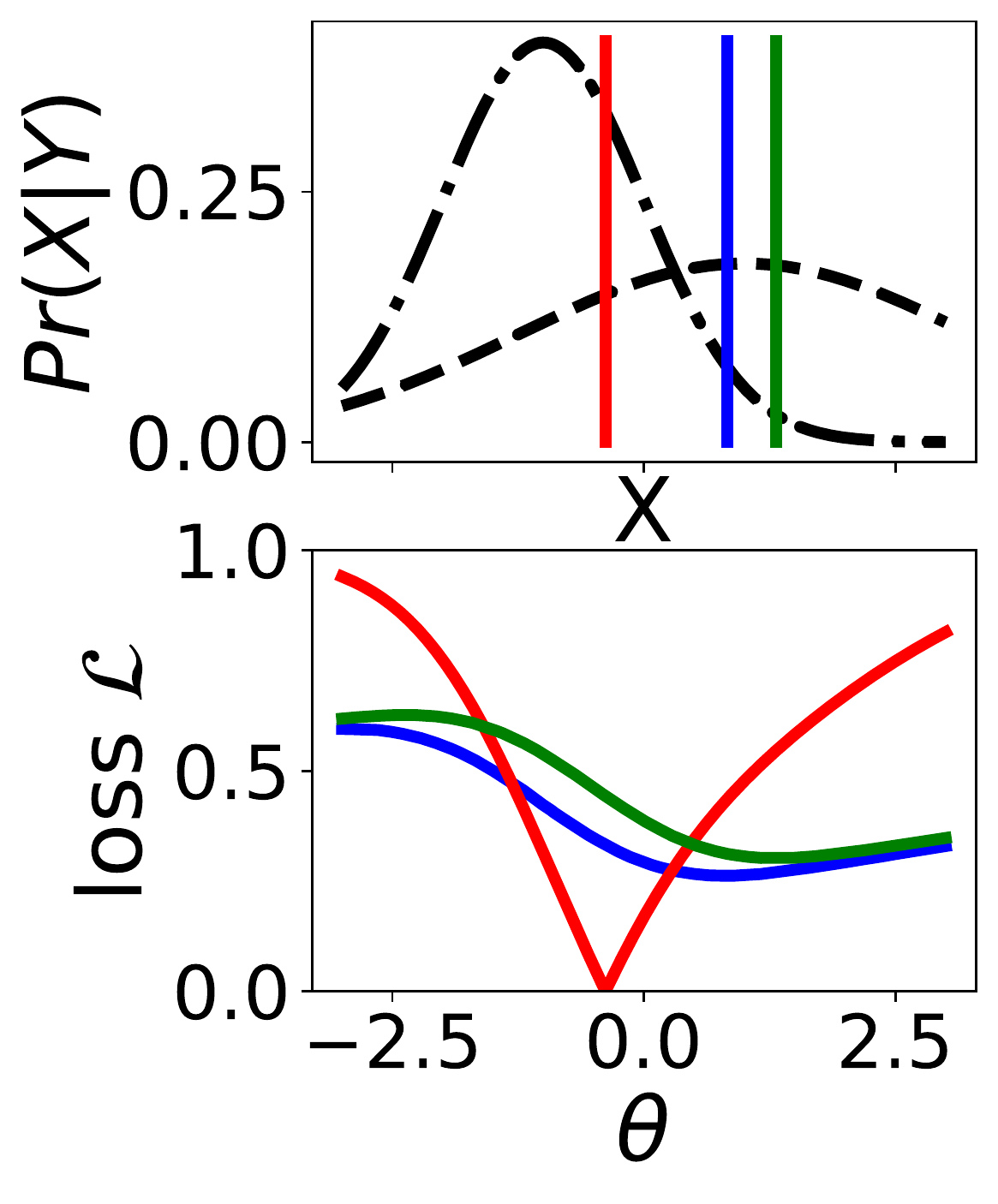}
\includegraphics[width=0.48\linewidth,height=130pt,valign=M]{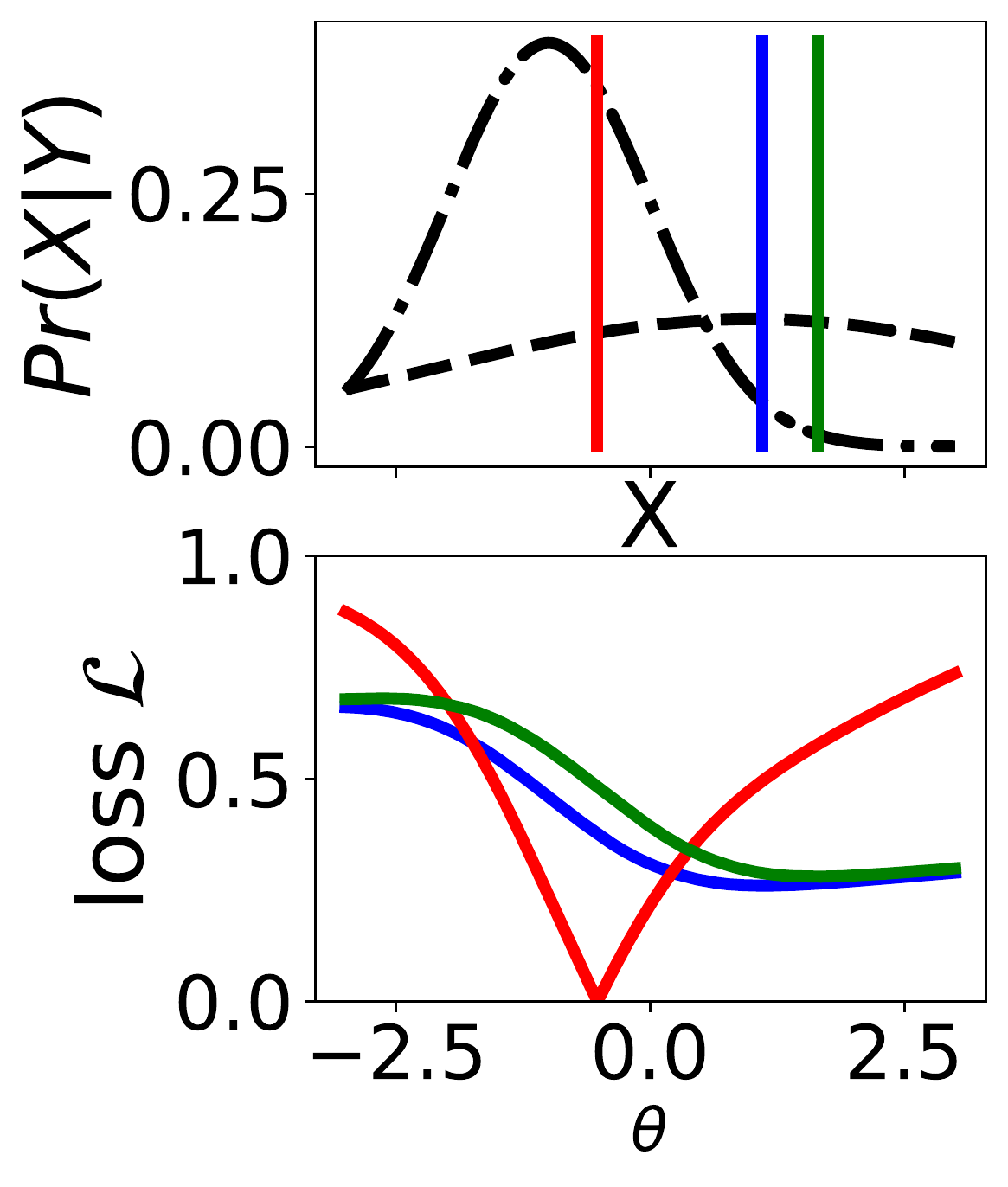}
\caption{Illustration of optimal natural ${\theta^*}$, fair $\theta_f$, and robust $\theta_{r}$ classifiers for $K=5$ (left) and $K=10$ (right) with $\mu_-=-1$ and $\mu_+=1$.
}
\label{fig:example}
\end{figure}

%%%%%%%%%%%%%%%%%%%%%%%%%%%%%%%%%%%%%%%%%%%%%%%%%%%%%%%%%%%%%%%%%%%%%%%%%%%%%%%%%%%%%%%%%%%%
\section{Why Fairness Weakens Robustness?}
\label{sec:theory}
%%%%%%%%%%%%%%%%%%%%%%%%%%%%%%%%%%%%%%%%%%%%%%%%%%%%%%%%%%%%%%%%%%%%%%%%%%%%%%%%%%%%%%%%%%%%
This section presents the main results of the paper. It will show that fairness affects model robustness because the learned decision boundary is {\em pulled in opposite directions} by fair and robust models. \textbf{} 
To render the analysis tractable, the theoretical discussion focuses on linear classifiers, and more specifically on learning a mixture of Gaussians with a linear classifiers. In addition, Section \ref{sec:experiments} will show that a similar phenomenon occurs in large non-linear models. 
%%%%%%%%%% Duplicate Definition %%%%%%%%%%%%%%%
% For linear models of the form $f_{\theta}(X) =\theta^\top X + b$ with $b \in \RR$, the \emph{distance to the 
% decision boundary} is implemented as:
% \begin{equation}\label{eq:dist}
%     % \dist{X}{\theta} = \frac{|\theta^\top X + b | }{\| \theta \|}.
%     \begin{aligned}
%        \dist{X}{\theta} =\max~&\epsilon\\
%        \text{s.t.}~& f_{\theta}(X+\tau) = f_{\theta}(X)\,, & \forall~\lVert\tau\rVert\leq \epsilon\,.
%     \end{aligned}
% \end{equation}
\noindent
This section assumes that $\cA = \cY$, i.e., the protected attribute is also the output of the classifier, again to simplify exposition. All proofs are given in Appendix \ref{app:proofs}. 

%% Building theo. intutions to explain the factors that control fairness and robustness. We are using mix Gaus but the results will extend empirically on general classifiers. 
%% Fig. 1 here and exaplain it.

\subsubsection*{Optimal Models for Mixtures of Gaussians}
Consider a binary classification setting  (i.e., $\cY = \{ -1, 1\}$) with data drawn from a 
mixture of Gaussian distributions, so that $\pr{X\mid Y=-1} \!\propto\! \cN(\mu_-, 1)$ and 
$\pr{X\mid Y=1} \!\propto\! \cN(\mu_+,K^2)$, with $\mu_-<\mu_+$ and different variances ($K>1$).  
The analysis can be easily extended to higher-dimensional cases, but these non-restrictive assumptions help simplifying and clarifying exposition. 
An illustration of this setting is reported in Figure \ref{fig:example} (top) where the data distributions are highlighted with black dashed curves. 

%\keyu{It might be worth mentioning that the analysis here can be easily extended or generalized to the $n$-dimensional case.}
% \keyu{In the theoretical studies I have done, I assume that $X\mid Y=1$
% follows a normal distribution $\cN(1,K^2)$ rather than $\cN(1,K)$. If the assumption I made
% is taken, Figure \ref{fig:example} should be modified accordingly.
% For the moment, I will stick with $\cN(\mu_+, K^2)$.}

The following analysis poses no restrictions on the relative subgroup sizes 
$|D_{1}|$ and $|D_{-1}|$ and focuses on the less-restrictive \emph{balanced} data setting, in which data samples from different protected groups are equally likely. 
% In addition, this section studies the balanced setting where 
% both classes are equally likely to show up, i.e., $\pr{Y=-1} = \pr{Y=1} = 1/2$.
%In this scenario where feature vectors are one-dimensional, this work 

The paper studies a family of parametric classifiers $\{f_{\theta}\}_{\theta}$ with $\theta \in[\mu_-,\mu_+] \subseteq \mathbb{R}$, where $f_{\theta}(X) = \mathbbm{1}\{ X > \theta \}$ denotes the classification output of the classifier. 
The optimal models with respect to the natural, fair, and robust losses can be specified as follows:
\\
$\bullet$ \textbf{Optimal natural model $\left(f_{\optclf}\right)$.} It is the Bayes classifier which minimizes the natural classification error as defined in Equation \eqref{eq:erm}.
In Figure \ref{fig:example} (top), this classifier is represented by vertical blue lines.
\\    
$\bullet$ \textbf{Optimal fair model $\left(f_{\fairclf}\right)$.} Intuitively, this classifier is $\fairclf(\infty)$ as defined in Equation \eqref{eqn:fair_reg}. Formally speaking, this classifier minimizes a lexicographic function whose first component is $\left\vert \sum_{a \in \cA } \!\!\left(\cL_{\theta}(D_a) -  \cL_{\theta}(D) \right) 
                          \right\vert$ and second component is $\cL_{\theta}(D)$. 
In Figure \ref{fig:example} (top), this classifier is represented by vertical red lines.
\\
$\bullet$ \textbf{Optimal robust model $(f_{\robclf})$.} This classifier minimizes the robust classification 
error in Equation \eqref{eq:robloss}, for a given $\epsilon$. 
In Figure \ref{fig:example} (top), it is depicted by vertical green lines.
    
\subsubsection*{Relationships Between the Optimal Models}
The next result characterizes the positional relationship among the three optimal models  mentioned above, which can be observed in Figure \ref{fig:example}.

\begin{theorem}\label{prop:three_clfs}
    For any $\epsilon\in\left[0,\frac{\pnDist}{2}\right]$
    and $K\in\left(1,\Kbound\right]$, where $\Kbound= 
        \min\left\{\exp\left(\frac{(\pnDist-2\epsilon)^2}{2}\right),\frac{\pnDist}{\epsilon}-1\right\},
    $
    \begin{equation}\label{eq:hierarchy}
    \mu_-+\epsilon\leq\fairclf \leq\; \optclf \;\leq \robclf\leq \mu_+-\epsilon\,.
    \end{equation}
    Besides, $\robclf$ is an increasing function of $\epsilon$
    over $\left[0,\frac{\pnDist}{2}\right]$.
\end{theorem}

\noindent
The result follows from the observation that
the optimal natural model $f_{\optclf}$ can be expressed as
\begin{multline*}
    \optclf=\mu_--\frac{\pnDist}{K^2-1}+\\
        \frac{K}{K^2-1}\sqrt{2(K^2-1)\ln(K)+(\pnDist)^2}\,;
\end{multline*}
the fair classifier $f_{\fairclf}$ as:
\begin{equation*}
        \fairclf =\mu_-+ \frac{\pnDist}{K+1}\,
\end{equation*}
and the robust classifier $f_{\robclf}$ as
\begin{multline*}
    \robclf=\mu_--\frac{\pnDist-(K^2+1)\epsilon}{K^2-1}+\\
    \frac{K}{K^2-1}\sqrt{2(K^2-1)\ln(K)+
    (\pnDist-2\epsilon)^2}\,.
\end{multline*}

\noindent
From the result above, it follows that 
{\bf (1)} {\em the fair classifier achieves the largest robust error while the robust classifier results in the least error}, 
and
{\bf (2)} {\em the fair classifier achieves the largest boundary error while the robust classifier results in the smallest boundary error}, as expressed by the following Corollaries.

\begin{corollary}
For any $\epsilon\in\left[0,\frac{\pnDist}{2}\right]$ and $K\in\left(1, \Kbound\right]$, 
\begin{equation*}
    \roberr{\fairclf}\geq \roberr{\optclf}\geq\roberr{\robclf}.
\end{equation*}
\end{corollary}

\begin{corollary}
For any $\epsilon\in\left[0,\frac{\pnDist}{4}\right]$ and $K\in\left(1, \newKbound\right]$,
    \begin{equation*}
        \bdyerr{\fairclf}\geq\bdyerr{\optclf}\geq
        \bdyerr{\robclf}\,,
    \end{equation*}
where $\newKbound
\!=\! \min\left\{\exp
\left(\!\frac{(\pnDist-2\epsilon)^2}{2}\!\right)\!,
\phi^{-1}\!\!\left(\!\frac{\pnDist}{\epsilon}-2\right)\!\right\}
$
and $\phi^{-1}$ is the inverse function associated with $\phi:[1,+\infty)\mapsto[2,+\infty)$ such that $\phi(x)=x+1/x$.
\end{corollary}

\noindent These results highlight the impossibility of achieving fairness and robustness simultaneously in this classification task. Fairness and robustness are pulling the classifier in opposite directions.

\subsubsection*{The Role of the Decision Boundary}
Building on the previous results, this section provides the key theoretical intuitions to explain why fairness increases adversarial vulnerability. It identifies the average distance to the decision boundary as the central aspect linking fairness and robustness, which is formalized in Theorem \ref{thm:avg_dist}. 

\begin{theorem}\label{thm:avg_dist}
For any $\epsilon\in\left[0,\frac{\pnDist}{2}\right]$ and $K\in\left(1, \Kbound\right]$,
\begin{equation*}
        \expect{\dist{X}{\robclf}}\geq\expect{\dist{X}{\optclf}}\geq\expect{\dist{X}{\fairclf}}. 
    \end{equation*}
In addition, the fair model minimizes the average distance to its decision boundary over all valid classifiers, i.e.,
    \begin{equation*}
        \fairclf=\argmin_{\theta\in[\mu_-,\mu_+]}~\expect{\dist{X}{\theta}}\,.
    \end{equation*}
\end{theorem}

\noindent
This result indicates that, among the three considered optimal models, the fair model has the smallest average distance to the decision boundary. while the robust model has the largest distance. The result above is exemplified in Figure \ref{fig:example}.  The bottom plots show the losses associated with the optimal natural, fair, and robust models for two choices of $K$ (left and right) while the top plots show the optimal decision boundaries associated with each of the three models -- notice that they correspond to the minima of their relative losses. 

Observe how class class $Y=1$ has a higher classification error than class $Y=-1$ under the natural (and thus unfair) classifier $f_{\optclf}$. This is intuitive since 
the conditional distribution $\pr{X\mid Y=1}$ has much higher variance than $\pr{X\mid Y=-1}$. Hence, to balance the classification errors, the fair classifier pushes the decision boundary towards the mean of class $Y=-1$. This increases the error of class $Y=-1$ while decreasing the error of class $Y=1$. In contrast, the robust classifier pushes the decision boundary far away from the dense input region, i.e., the mean of the data associated with class $Y=-1$. 
%Thus robustness constraint pushes the decision boundary far away from samples of class $Y=-1$. 

There are a few points worth emphasizing. First, \emph{robustness and fairness pull the decision boundary into two opposite directions}. Second, the fair model  $f_{\fairclf}$ results in predictions with higher robust errors, when compared to the optimal natural model $f_{\optclf}$, and it also increases adversarial vulnerability as the variance $K$ increases. The variance $K$ regulates the difference in the standard deviation of the underlying distributions associated with the protected groups and thus controls the overall distance to the decision boundary.
 %We provide a theoretical justification to explain why fair models are more vulnerable than unfair models in this section. 
{\em In summary, fairness can reduce the average distance of the training samples to the decision boundary which, in turn, makes the model less tolerant to adversarial noise}. 

This section concludes with another important result. The previous relationships continue to hold even when the optimality conditions of the fair classifier are relaxed, i.e., when $\lambda$ is taking values different from $\infty$. Moreover, the fairness constraints always reduce the distance to the decision boundary among protected groups and this reduction is proportional to the strength of the fairness constraints (or the tightness of the required fairness bound $\alpha$). 

% We follow the same settings in previous example by assuming $X\mid Y=-1 \sim \mathcal{N}(\mu_{-}, \sigma^2 I)$ and $X\mid Y=1 \sim \mathcal{N}(\mu_{+}, K^2 \sigma^2 I)$. Here  we consider the balanced settings when $\pr{Y=1} = \pr{Y=-1} = \frac{1}{2}$
\begin{theorem}
\label{thm:6.5}
Consider the fair classifier $f_{\fairclf (\lambda)}$ that optimizes Equation \eqref{eqn:fair_reg}. 
It follows that, 
for any $\lambda \in \left(\frac{K-1}{K+1}, +\infty\right)$, 
\[
\fairclf(\lambda)  =\fairclf\]
while for any $\lambda \in\left[0, \frac{K-1}{K+1}\right]$,
%\begin{align*}
\(
    \fairclf(\lambda)=\mu_--\frac{\pnDist}{K^2-1}+
        \frac{K}{K^2-1}\sqrt{2(K^2-1)\ln\left(\frac{1-\lambda}{1+\lambda}\cdot 
 K\right)+(\pnDist)^2}
\).
%\end{align*}
 Moreover, the parameter $\theta$ associated with the fair classifier and 
 the average distance to its decision boundary $\expect{\dist{X}{\fairclf(\lambda)}}$
 are both decreasing as $\lambda$ increases.
\end{theorem}
\noindent
Informally speaking, Theorem \ref{thm:6.5} states that applying fairness constraint with large enough penalty $\lambda$ will push the decision boundary towards the negative class (group with smallest variance). As a result, the average distance to the decision boundary of all samples will be reduced. 

While the analysis above applies to the linear setting considered in this section, the results were empirically 
validated on large non-linear models. For example, Figure \ref{fig:fair_UTK} compares the performance of a 
penalty based fair CNN model (bottom plots) with $\lambda=1.0$ against a natural (non-fair) CNN classifier (top plots). 
The left plots report the task accuracy by each subgroup (denoting {races}) and average distance to decision boundary (right) of each subgroup. Note how the fair classifier reduces the disparities in task accuracy experienced by the various subgroups. This effect, however, also reduces the \emph{overall} average distance to the decision boundary. As a consequence, 
fair models will be more vulnerable to adversarial perturbations. 

The next sections focus on assessing these theoretical intuitions onto general non-linear classifiers in a variety of settings and on devising a possible mitigation strategy to balance a good tradeoff between fairness and robustness.

% Note that, while for simplicity this section assumes the data follow a mixture of Gaussians, the paper empirically 
% validates the impact of fairness constraints in reducing the average distance to the decision boundary in 
% nonrestrictive settings using large non-linear models and differentiable loss. For example, we report in Figure \ref{fig:fair_UTK} the impact of Fair-Reg with $\lambda=1.0$ (bottom sub figures) towards the group accuracy and avg distance to the decision boundary compared to unfair models (top sub figures). We can see that while applying fairness constraint in Fair-Reg can reduce the accuracy gap it has side effects in reducing the overall average distance to the decision boundary of samples. As a result, a small perturbation added to  testing samples can easily make fair classifier mis-classify these samples. 

 \begin{figure}[t]
 \centering
\includegraphics[width=\linewidth,height=100pt]{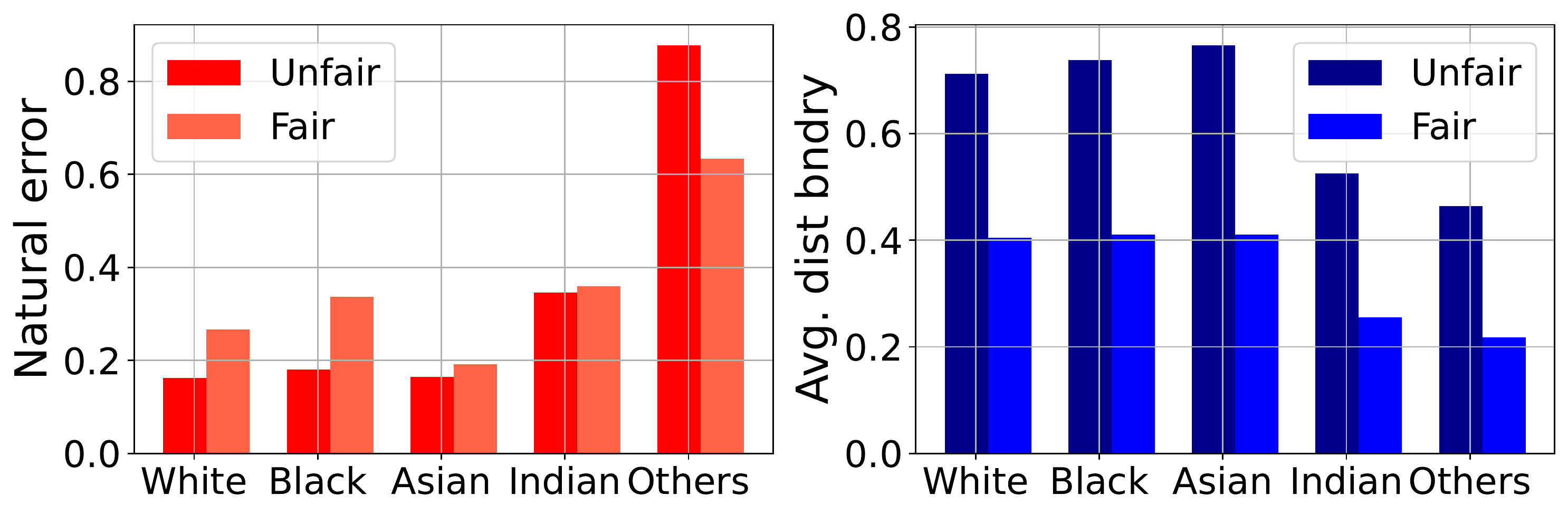}
\caption{Comparison between group's natural accuracy (left) and group's average distance to the decision boundary (right) between unfair and fair models on the UTK-Face dataset.
}
\label{fig:fair_UTK}
\end{figure}

%%%%%%%%%%%%%%%%%%%%%%%%%%%%%%%%%%%%%%%%%%%%%%%%%%%%%%%%%%%%%%%%%%%%%%%%%%%%%%%%%%%%%%%%%%%%%
\section{Beyond the Linear Case}
\label{sec:experiments}
%%%%%%%%%%%%%%%%%%%%%%%%%%%%%%%%%%%%%%%%%%%%%%%%%%%%%%%%%%%%%%%%%%%%%%%%%%%%%%%%%%%%%%%%%%%%%

This section validates  the theoretical intuitions presented above on much more complex architectures, datasets, and loss functions. 
The experiments focus on highlighting fairness, robustness, errors, and their relation to the distance to the decision boundary. 
When $f_\theta$ is a non-linear model, computing the distance to the decision boundary becomes a computational challenge. Thus, this section uses a commonly adopted proxy metric that measures the difference between the first two order statistics of the softmax outputs in the model \cite{wang2022enhancing,lecuyer2019certified}.
% \begin{equation}
% \label{eq:dist_nl}
%     \mbox{dist}(X, \theta) = p^{\max}_{\theta}(X) -  p^{\text{2nd\_max}}_{\theta}(X),
% \end{equation}
% where $p_{\theta}(X) = \mbox{softmax}(f_{\theta}(X))$.

\noindent\textbf{Datasets.}
The experiments of this section focus on three vision datasets: \textit{UTK-Face} \cite{zhifei2017cvpr}, \textit{FMNIST} \cite{xiao2017fashion} and 
\textit{CIFAR-10} \cite{krizhevsky2009learning}. 
The adopted protected groups and labels in the UTK-Face datasets are {\sl ethnicity} (White/Black/Indian/Asian/Others) or {\sl age} (nine age bins), resulting in two distinct tasks. 
For FMNIST and CIFAR, the experiments use their standard labels and assume that labels are also protected groups, mirroring the setting of  previous work  \cite{xu2021robust,cuong_neurips22,robust_fair}. 
A complete description of the dataset and settings is found in Appendix \ref{app:experiments}.

\noindent\textbf{Settings.} The experiments consider several deep neural network architectures, including CNN \cite{o2015introduction}, ResNet 50 \cite{he2016deep} and VGG-13 \cite{simonyan2014very}. The former uses 3 convolutional layers followed by 3 fully connected layers. 
Models trained on the UTK-Face data use a learning rate of $1e^{-3}$ and 70 epochs. Those trained on FMNIST and CIFAR, use a learning rate of $1e^{-1}$ and 200 epochs, as suggested in previous work \cite{xu2021robust}. For all datasets and models, unless otherwise specified,  a batch size of 32 is used. 
The experiments analyze penalty-based fairness method, RFGSM attacks \cite{tramer2017ensemble}, and the VGG-13 network, unless specified otherwise. Additional experiments using group-loss focused method (see Appendix \ref{app:fairness_defs}), additional network architectures, and adversarial attacks are reported in Appendix \ref{app:experiments}.

\iffalse
The experiments vary the fairness parameters \rev{$\lambda \in [0, 10] $ and $q \in [3]$} to measure the impact of the fairness constraints strength towards fairness and robustness. 
{\sl Robust models} are obtained by implementing a project gradient descent scheme (PGD) \cite{madry2017towards}:
\begin{equation}
   \theta_{\mathrm{r}}^{\left(\epsilon^*\right)}= \underset{\theta}{\argmin}~\frac{1}{n}\sum^n_{i=1} \max_{\|\tau\|_p \leq \epsilon^* }\ell(f_{\theta}(X_i+\tau), Y_i) 
\end{equation}
where $\epsilon^* >0 $ denotes the desired robustness level.
\fi

% \textcolor{blue}{@Cuong: Check above. Is this what you really do for robust classifiers? Also what is value $p$ you choose? 
% My guess is that you did NOT use the above. => We do not implement any  solely robust classifier. In Section below, see equation (13) we implement a fair robust model that achieves both fairness and robustness at the same time. Where the norm $p \in \{0,2, \infty\}$. Since we use PGD so $p = \infty$.}

\begin{figure}[t]
 \centering
\includegraphics[width=\linewidth]{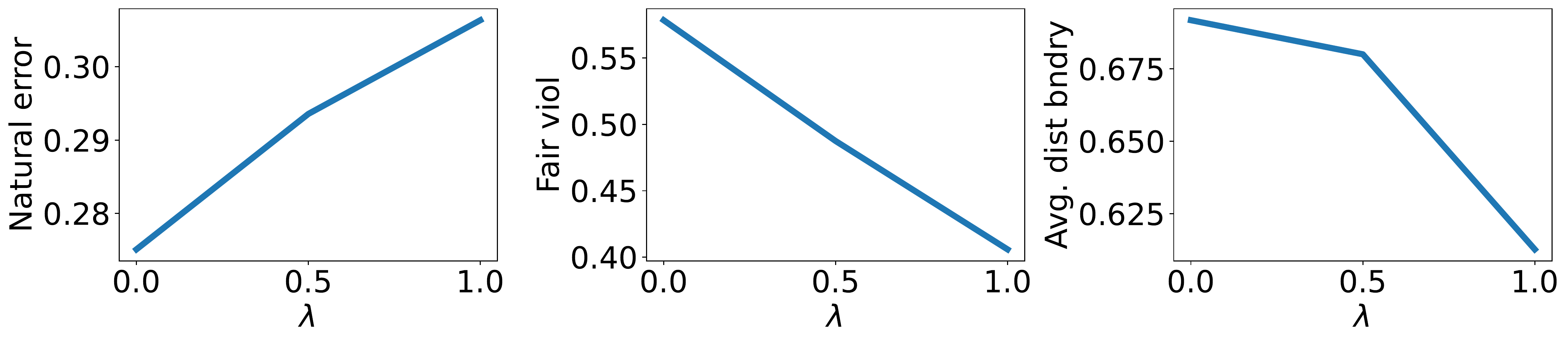}
\includegraphics[width=\linewidth]{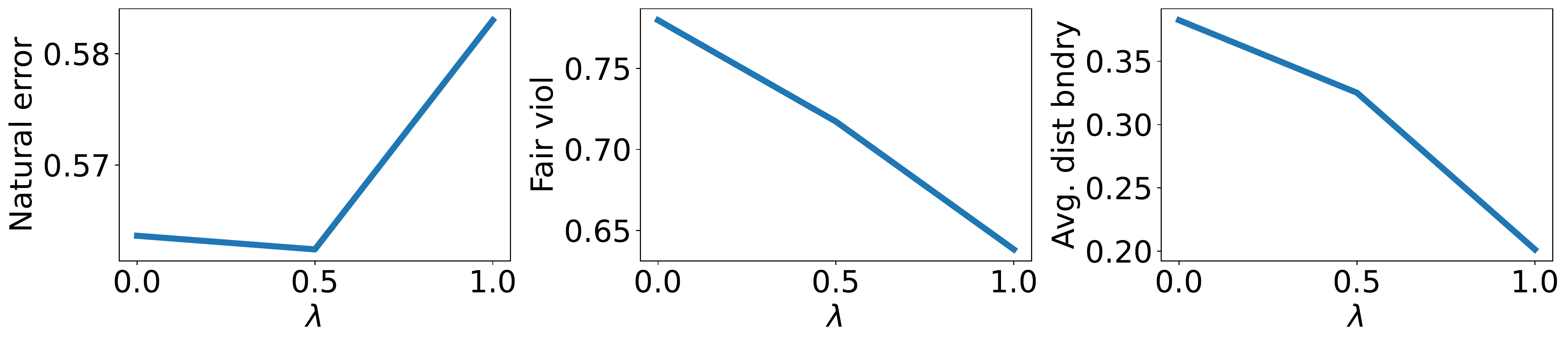}
\includegraphics[width=\linewidth]{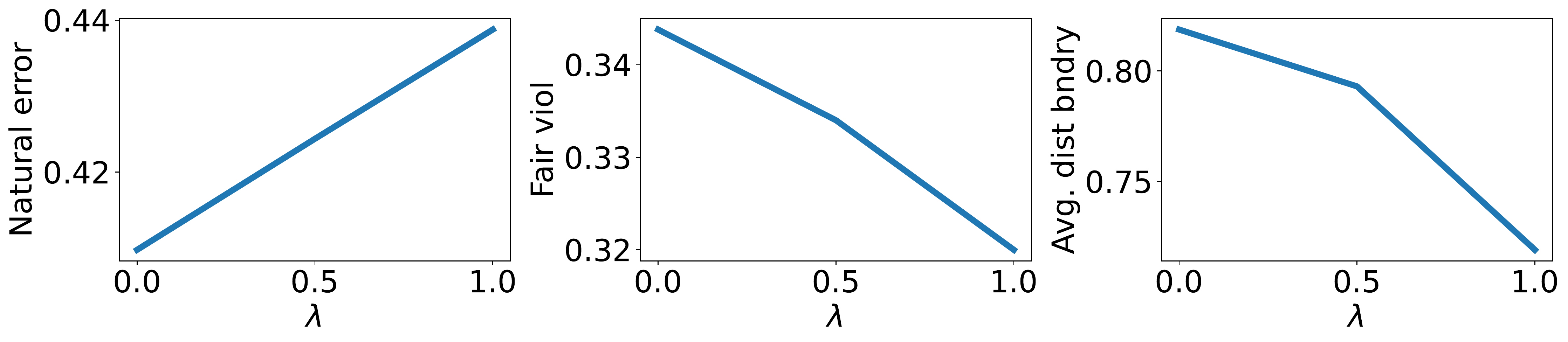}
\caption{Natural errors, fairness violations, and average distance to the decision boundary for the UTK-Face {\sl ethnicity} (top), 
UTK-Face {\sl age bins} (middle) and CIFAR (bottom) datasets when varying the fairness parameter $\lambda$ on a CNN model. 
}
\label{fig:avg_dist}
\end{figure}

\begin{figure*}[t]
\centering
\includegraphics[width=0.9\linewidth]{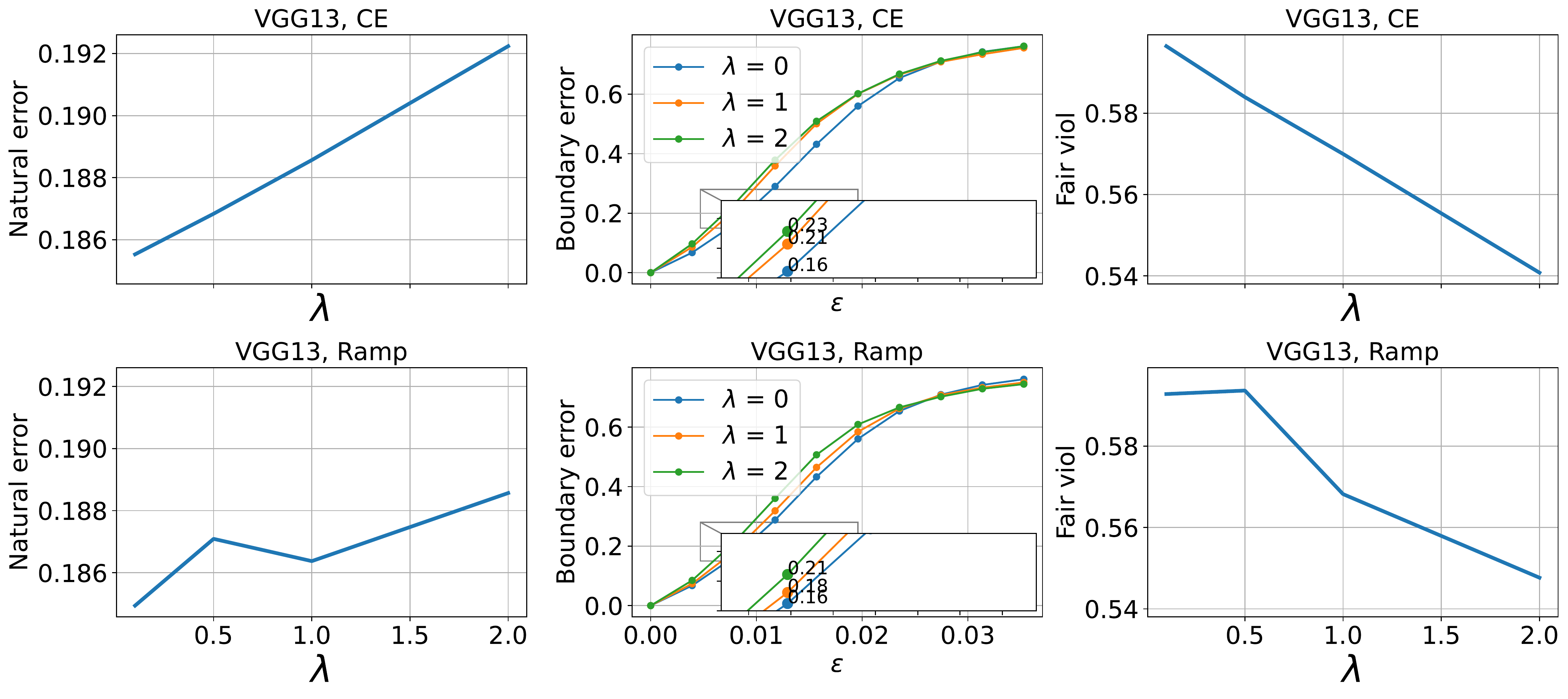}
\caption{{\bf Top}: Natural errors (left) and fairness violations (right) on the UTKFace {\sl ethnicity} task at varying of the fairness parameters $\lambda$. The middle plots compares the robustness of fair ($\lambda > 0)$ vs.~natural ($\lambda=0$) classifiers to different RFGSM attack levels. 
{\bf Bottom}: Mitigating solution using the bounded Ramp loss.}
\label{fig:bnd_err}
\end{figure*}

%%%%%%%%%%%%%%%%%%%%%%%%%%%%%%%%%%%%%%%%%%%%%%%%%%%%%%
\subsubsection*{Fairness impacts on the decision boundary}
%%%%%%%%%%%%%%%%%%%%%%%%%%%%%%%%%%%%%%%%%%%%%%%%%%%%%%

As shown by Theorem \ref{thm:6.5}, fairness reduces the average distance of the testing samples to the decision boundary. This section illustrates how this result carries over to larger non-linear models.
Figure \ref{fig:avg_dist} reports results obtained by executing the penalty-based fair models 
on the UTK-Face datasets for ethnicity (top) and age (middle) classification and on CIFAR (bottom). 
A clear trend emerges: As more fairness is enforced (larger $\lambda$ values), the natural errors (left plots) increase, while the fairness violations (center plots) decrease. Importantly, and in agreement with the theoretical results, the experiments report a sharp reduction to the average distance to the  decision boundary (right plots). This behavior renders fair models more vulnerable to adversarial attacks, as will be highlighted shortly. Similar results are also observed for the group-loss based models and other architectures. %(see Appendix \ref{app:experiments}).

%  \begin{figure*}[t]
%  \centering
% \includegraphics[width=0.9\linewidth,height=120pt]{fair_clf_SVHN_bs_512.pdf}
% \caption{Accuracy, fairness violation, and avg. distance to the decision boundary on SVHN dataset under  different fairness parameters $\lambda$ of Fair-Reg
% }
% \label{fig:avg_dist_svhn}
% \end{figure*}

%%%%%%%%%%%%%%%%%%%%%%%%%%%%%%%%%%%%%%%%%%%%%%%%%%%%%%
\subsubsection*{Boundary errors increase as fairness decreases}
%%%%%%%%%%%%%%%%%%%%%%%%%%%%%%%%%%%%%%%%%%%%%%%%%%%%%%

\begin{figure}[!t]
 \centering
\includegraphics[width=\linewidth]{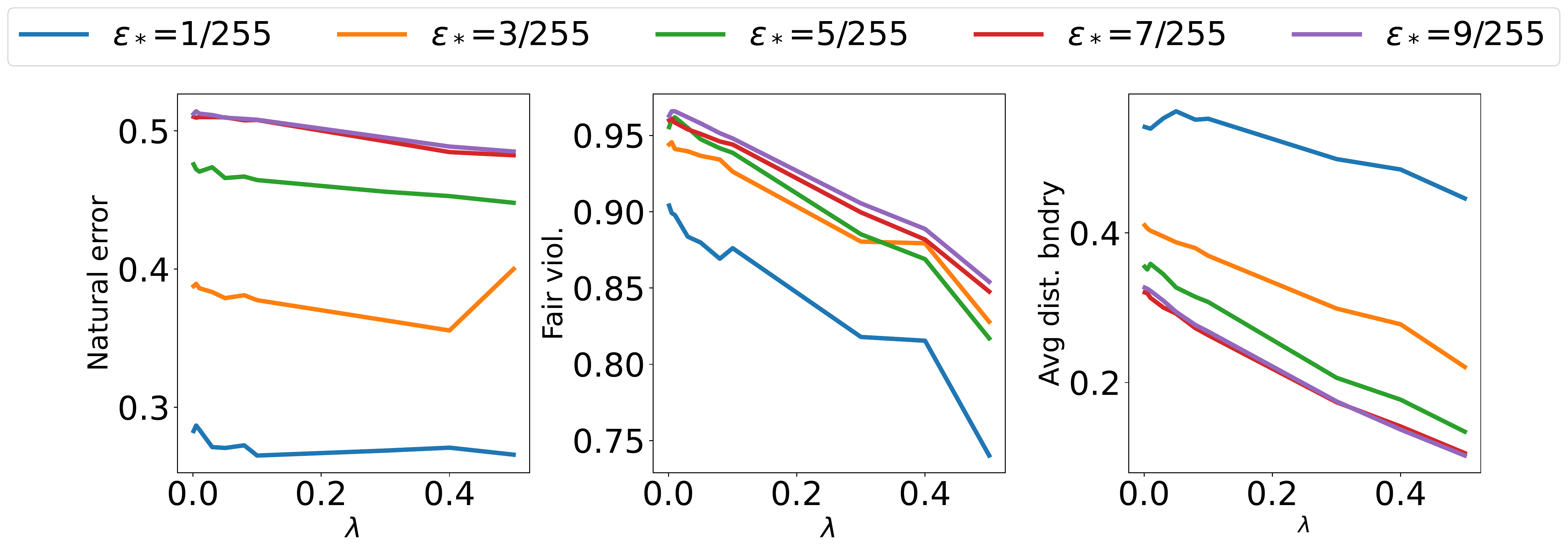}
\caption{ Natural error (left), fairness violation (middle) and average distance to the decision boundary (right) at varying of the margin perturbation $\epsilon_*$ and fairness parameters $\lambda$.
}
\label{fig:robust_fair_gen}
\end{figure}

This section highlights the key consequence of the sharp reduction to the average distance to the  decision boundary: \emph{the increase of the vulnerability to adversarial attacks}.  Figure \ref{fig:bnd_err} (top) reports the natural errors (left),  boundary errors (middle), and fairness violations (right)  for a VGG-13 model trained on UTKFace dataset on the {\sl ethnicity} task using a standard cross-entropy (CE) loss. 
Once again, other architectures\footnote{With the caveat that VGG-13 could not be used for FMNIST since the 28x28 pixel resolution of FMNIST  is smaller than that required by some VGG filters.} and datasets are reported in the appendix and the results follow the same trends as those reported here.

The natural errors and fairness violations are reported for \emph{fair} classifiers, at varying of the fairness violation parameter $\lambda$. The boundary errors (middle) are reported for classifiers satisfying various fairness levels (i.e., using different $\lambda$ values) and at varying of the strength $\epsilon$ of the desired robustness level (see Equation \eqref{eq:rob_error}). 

Notice how, compared to the natural models, the fair models incur much higher natural and boundary errors. In particular, the relative increase in boundary errors are significant: The fairness models have boundary errors that are up to  9\% larger than their natural counterparts. These observations match the theoretical analysis and highlight a significant increase in vulnerability to adversarial examples by the fair models, even for moderate selections of the fairness violation parameters $\lambda$.

 \begin{figure*}[!t]
 \centering
\includegraphics[width=\linewidth]{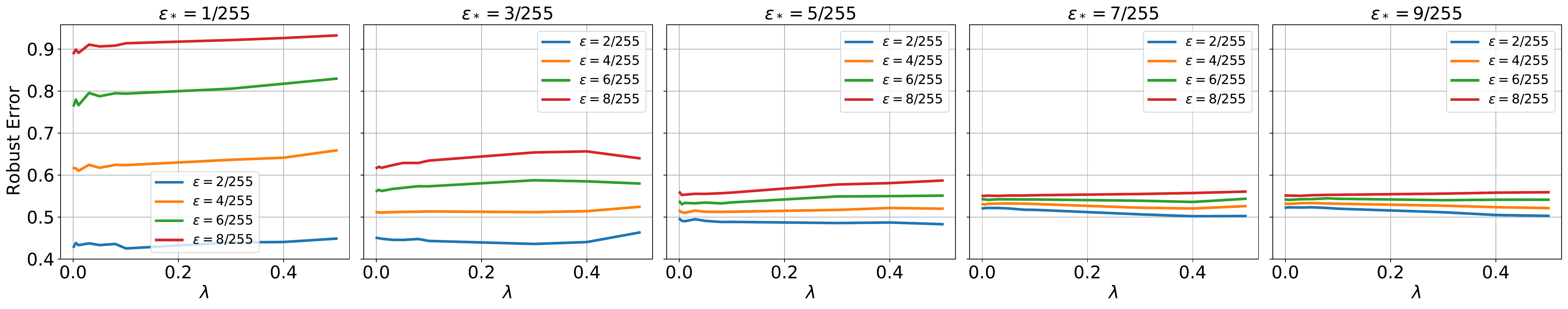}
\caption{Robust errors for different attack levels $\epsilon$ of a robust and fair classifier at varying of the margin perturbation $\epsilon_*$ and fairness parameters $\lambda$.
}
\label{fig:robust_fair_robust_err}
\end{figure*}

% \subsection{Analysis on minimum distance to generate adversarial samples}

\subsubsection*{Enforcing Fairness and Robustness Simultaneously}

This section considers an additional experiment to show how fairness may negatively impact robustness. 
It reports results of a classifier trying to enforce both fairness and robustness, similarly to the proposal of Xu et al.~\cite{xu2021robust}. The resulting model aims at solving the following regularized ERM problem:
\begin{align}
    &\min_{\theta} \frac{1}{n} \sum_{i=1}^n \max_{\| \tau \|_{p} \leq \epsilon_*} \!\!\ell(f_{\theta}(X_i + \tau), Y_i) \notag\\ 
    &+\lambda \left\vert\nicefrac{1}{|D_c|}
    \!\!\!\!\!\!\!\!\sum_{(X,A,Y) \in D_c}\!\!\!\!\!\!\!\! \ell(f_{\theta}(X), Y) - \frac{1}{n}\sum_{i=1 }^n \ell(f_{\theta}(X_i), Y_i)  \right\vert
        \label{eq:fair_robust_opt}
\end{align}
using stochastic gradient descent. The first component aims at increasing the robustness of the classifier under a margin perturbation $\epsilon_*$, following the PGD training \cite{madry2017towards} with perturbation norm $p=\infty$.
It works by first generating adversarial samples $X_i +\tau$, where $\|\tau\|_{\infty} \leq \epsilon_*$, and then the learning progress aims at minimizing the loss between the model prediction for that adversarial samples and the ground-truth $\ell(f_{\theta}(X_i + \tau), Y_i) $. 
The larger the margin perturbation $\epsilon_*$, the more robust the resulting classifier. The second component implements a penalty-based fairness strategy \cite{agarwal2018reductions}, which promotes fairness by penalizing the difference among each groups' average loss and the overall's average loss. 

The experiments vary the margin perturbation $\epsilon_*$ (robustness) and the penalty value $\lambda$ (fairness). 
Figure \ref{fig:robust_fair_gen} reports the (natural) error (left), fairness violations (middle) and average distance to the decision boundary (right) for different levels of the margin perturbation $\epsilon_*$ on the UTK-Face (ethnicity) dataset. 
As expected, enforcing larger margin perturbations $\epsilon_*$ increases the average distance to the decision boundary (thus improving robustness), but at the cost of significantly increasing the natural errors. Increasing the fairness parameter $\lambda$ decreases the average distance to the decision boundary). 

Figure \ref{fig:robust_fair_robust_err} reports the robust errors under different levels of adversarial attacks, which are specified by the level of perturbation $\epsilon$. Notice how the level of defense $\epsilon_*$ correlates with higher robustness (and smaller average distances to the decision boundary) for all fairness parameters $\lambda$ tested. These results show the challenge to achieving simultaneously robustness, fairness, and accuracy. 

{\em Overall, the results show that, without a careful consideration, inducing a desired equity property on a learning task may create significant security challenges.} This should not be read as an endorsement to satisfy a single property, but as a call for additional research at the intersection of fairness and robustness in order to design appropriate tradeoffs. 
% \textcolor{red}{Our results also emphasize that enforcing both fairness and robustness as proposed in \cite{xu2021robust} might not be an effective solution to balance the trade-off among accuracy, robustness and fairness as they claimed.}

\section{A Mitigating Solution with Bounded Losses}

\begin{figure}[t]
 \centering
\includegraphics[width=0.8\linewidth]{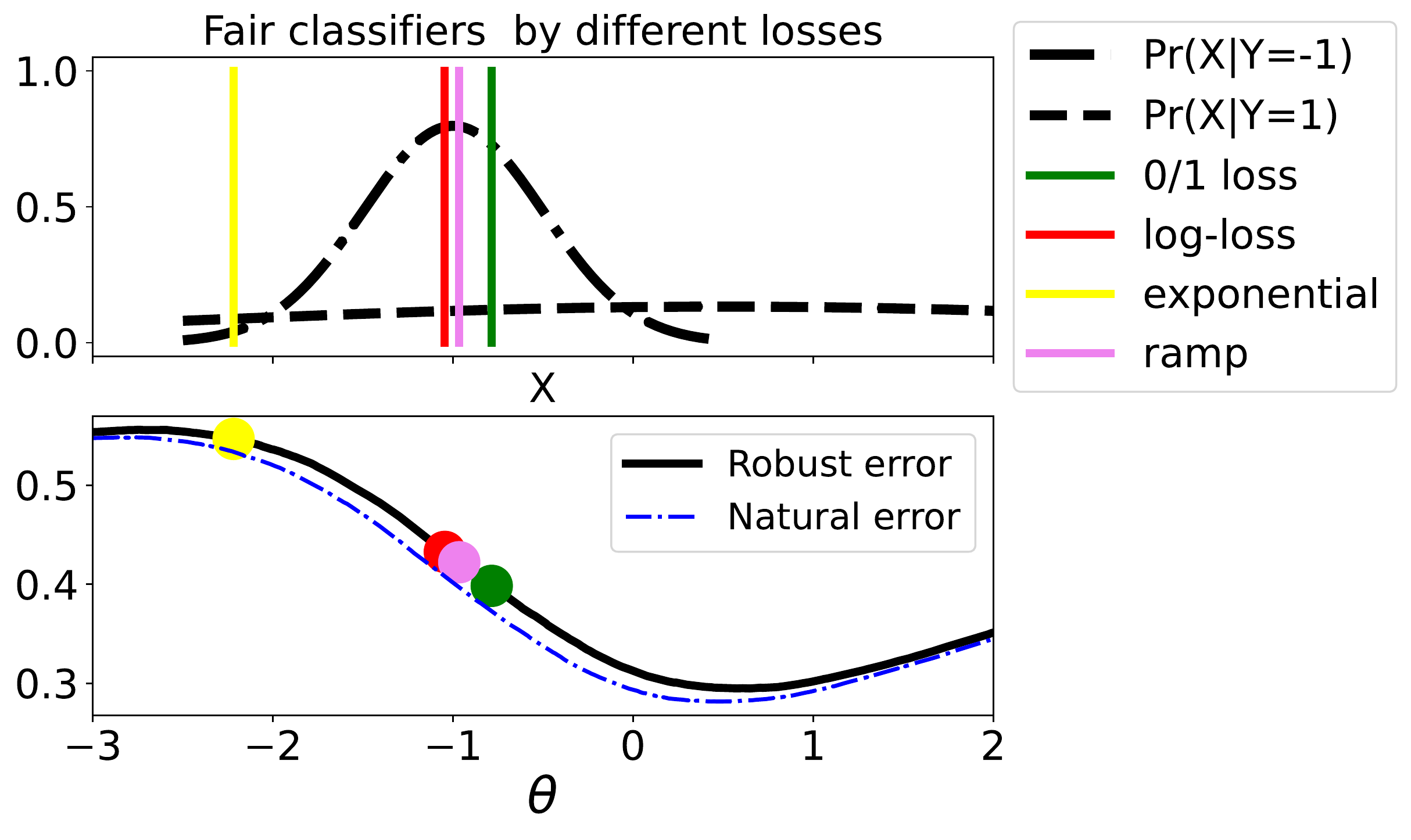}
\caption{Classifiers obtained using different loss functions (top) and the associated natural and robust error obtained by such losses.
}
\label{fig:impact_loss}
\end{figure}

While the previous sections have shown that the conflict between fairness and robustness is unavoidable, this section proposes a theoretically motivated solution attempting to attenuate this tension. The proposed solution relies on the observation that, using standard (unbounded) loss functions, misclassified samples lying far away from the decision boundary are associated to much larger losses than those which are closer to it. 
Recall also that the decision boundary was found as the predominant factor linking fairness and robustness. 
This key observation suggests the use of a bounded loss function, defined as \cite{goh2016satisfying,collobert2006trading}: 
\begin{equation*}
\ell_{Ramp}(f_{\theta}(X),Y)=  \min(1, \max(0, 1- Yf_{\theta}(X))).
\label{eqn:ramp}
\end{equation*}
and referred to as \emph{Ramp loss}, with domain $(0,1]$. The proposed strategy simply applies this loss function to a fair classifier (Equation \ref{eqn:fair_reg}). Its benefits can be appreciated in Figure \ref{fig:impact_loss}, which reports the results for the same setting used in the previous section and compares a fair classifier trained using the ramp loss with one trained using a 0/1-loss (which is also bounded but not differentiable), a log-loss, and an exponential loss (both unbounded) (top). 
The results show that the fair classifier trained using a ramp-loss is the least impacted by misclassified samples, resulting in lower robust errors compared to unbounded losses. It can be observed in the bottom subplot, where its associated loss is the closest, among all differentiable losses, to the local minima. 
The observed benefits of the ramp loss also carry over high-dimensional data and non-linear models, as shown in Figure \ref{fig:bnd_err} bottom, and further reported in Appendix \ref{app:experiments}. 

%Notice how the fair classifier$\theta_f$ aligns $\expect{\ell(f_{\fairclf}(X),Y)\mid Y=1)} = \expect{\ell(f_{\fairclf}(X),Y)\mid Y=-1)}$ for different choice of $\ell(.)$. Note that only 0/1 loss and ramp loss are bounded losses in this example but 0/1 loss is not differentiable to use in any model training. 

% . Figures \ref{fig:utk_age_bins_vgg} and \ref{fig:avg_dist_utk_races} compare the boundary error differences between fair and unfair classifiers on different datasets when using cross entropy loss (CE) and Ramp loss. The figures report boundary errors for UTK-age bins and UTK-races in Similar results are given for CIFAR-10 in Figure \ref{fig:cifar_vgg} and FMNIST in Figure \ref{fig:fmnist_vgg}. It is clear that, by using the bounded loss such as the Ramp loss, the boundary error differences between fair and unfair classifiers are reduced significantly.

\section{Conclusions}
This paper was motivated by two key challenges brought by the the adoption of  modern machine learning systems in consequential domains: \emph{fairness} and \emph{robustness}. 
The paper observed and analyzed the relationship between these two important machine-learning properties and showed that fairness increases vulnerability to adversarial examples. 
Through a theoretical analysis on linear models, this work provided a new understanding of why such tension arises and identified the distance to the decision boundary as a key explanation factor linking fairness and robustness. 
These theoretical findings were validated on non-linear models through extensive experiments on a variety of vision tasks. Finally, building from this new understanding, the paper proposed a simple, yet effective, strategy to find a better balance between accuracy, fairness and robustness. We hope these results could stimulate a needed discussion and research at the intersection of fairness and robustness to achieve appropriate tradeoffs.

\section*{Acknowledgments}
This research is partially supported by grants NSF-2112533, NSF-2133169 and NSF-2143706. 
F.~Fioretto is also supported by a Google Research Scholar Award and an 
Amazon Research Award. Its views and conclusions are those of the authors only.

%%%%%%%%% REFERENCES
{\small
\bibliographystyle{ieee_fullname}
\bibliography{lib}
}

\newpage
\appendix

\section{Missing Proofs}
\label{app:proofs}

\textbf{Proof of Theorem 6.1}

\begin{proof}
    % Let $\Phi$ denote the cumulative distribution function
    % associated with the standard normal distribution.
    \begin{enumerate}[i)]
        \item 
        Notice that the natural classification error and its derivative can be expressed
        as 
        \begin{align*}
            \naterr{\theta}=~&\pr{f_{\theta}(X)\neq Y}\\
            =~&\frac{1}{2}\pr{f_{\theta}(X)\neq 1\mid Y=1}+
        \\
        &\frac{1}{2}\pr{f_{\theta}(X)\neq -1\mid Y=-1}\\
        =~&\frac{1}{2}\int_{-\infty}^{\theta}\frac{1}{\sqrt{2\pi}K}\exp\left(-\frac{(x-\mu_+)^2}{2K^2}\right)dx+\\
        &\frac{1}{2}\int_{\theta}^{+\infty}\frac{1}{\sqrt{2\pi}}\exp\left(-\frac{(x-\mu_-)^2}{2}\right)dx\,.
        \end{align*}
        and
        \begin{equation*}
            \left(\naterr{\theta}\right)'=\frac{\left[\exp\left(-\frac{(\theta-\mu_+)^2}{2K^2}\right)-K\exp\left(-\frac{(\theta-\mu_-)^2}{2}\right)\right]}{2K\sqrt{2\pi}}\,.
        \end{equation*}
        The derivative $\left(\naterr{\theta}\right)'$ turns out to be an
        increasing function over the interval $[\mu_-,\mu_+]$ with
        $\left(\naterr{\mu_-}\right)'<0$ and $\left(\naterr{\mu_+}\right)'>0$ due
        to the assumption that $K<\Kbound<\exp\left(-\frac{(\pnDist)^2}{2}\right)$.
        
        Since the Bayes classifier is to minimize the 
        natural classification error, $\optclf$ 
        is supposed to be the unique root of $\left(\naterr{\theta}\right)'$, i.e.,
        \begin{equation*}
            \frac{\left[\exp\left(-\frac{(\optclf-\mu_+)^2}{2K^2}\right)-K\exp\left(-\frac{(\optclf-\mu_-)^2}{2}\right)\right]}{2K\sqrt{2\pi}}=0\,.
        \end{equation*}
        By solving the equation above, we end up with the following.
        \begin{multline*}
            \optclf=\mu_--\frac{\pnDist}{K^2-1}+\\
            \frac{K}{K^2-1}\sqrt{2(K^2-1)\ln(K)+(\pnDist)^2}\,,
        \end{multline*}
        which belongs to the open interval $(\mu_-,\mu_+)$.
        \item Equalized classification errors require the following equations hold.
        \begin{align*}
             &\pr{f_{\fairclf}(X)\neq 1\mid Y=1}\\
             =~&\int_{-\infty}^{\fairclf}\frac{1}{\sqrt{2\pi}K}\exp\left(-\frac{(x-\mu_+)^2}{2K^2}\right)dx\\
             =~&\int_{\fairclf}^{+\infty}\frac{1}{\sqrt{2\pi}}\exp\left(-\frac{(x-\mu_-)^2}{2}\right)dx\\
             =~&\pr{f_{\fairclf}(X)\neq -1\mid Y=-1}\,,
        \end{align*}
        which leads to the result that
        \begin{equation*}
            \fairclf=\mu_-+\frac{\pnDist}{K+1}>\mu_-+\epsilon\,,
        \end{equation*}
        where the inequality is due to the assumption that $K<(\pnDist)/\epsilon-1\leq \Kbound$.
    \item The robust classification error and its partial derivative can then be given by the following:
    \begin{align*}
     &\roberr{\theta}=\pr{\exists~\vert\tau\vert\leq \epsilon,~f_{\theta}(X+\tau)\neq Y}\\
            =~&\frac{1}{2}\pr{\exists~\vert\tau\vert\leq \epsilon,~f_{\theta}(X+\tau)\neq 1\mid Y=1}+\\
        &\frac{1}{2}\pr{\exists~\vert\tau\vert\leq \epsilon,~f_{\theta}(X+\tau)\neq -1\mid Y=-1}\\
        =~&\frac{1}{2}\left(\pr{X\leq \theta+\epsilon\mid Y=1}+\pr{X>\theta-\epsilon\mid Y=-1}\right)\\
        =~&\frac{1}{2}\int_{-\infty}^{\theta+\epsilon}\frac{1}{\sqrt{2\pi}K}\exp\left(-\frac{(x-\mu_+)^2}{2K^2}\right)dx+\\
        &\frac{1}{2}\int_{\theta-\epsilon}^{+\infty}\frac{1}{\sqrt{2\pi}}\exp\left(-\frac{(x-\mu_-)^2}{2}\right)dx\,,
    \end{align*}
    and
    \begin{multline}\label{eq:rob_opt}
        \frac{\partial}{\partial \theta}\roberr{\theta}=\\\frac{\left[\exp\left(-\frac{(\theta+\epsilon-\mu_+)^2}{2K^2}\right)-K\exp\left(-\frac{(\theta-\epsilon-\mu_-)^2}{2}\right)\right]}{2K\sqrt{2\pi}}\,.
    \end{multline}
    By the assumptions made about $K$ and $\epsilon$,
    there exists a unique root $\robclf\in(\mu_-,\mu_+-\epsilon)$ of
    $\frac{\partial}{\partial \theta}\roberr{\theta}=0$ such
    that the robust error $\roberr{\theta}$ is decreasing over $(\mu_-,\robclf)$ while increasing over $(\robclf,\mu_+)$, i.e.,
    \begin{equation}\label{eq:monotonicity}
        \frac{\partial}{\partial \theta}\roberr{\theta}\begin{cases}
        <0\,, &\theta\in(\mu_-,\robclf)\,,\\
        >0\,, &\theta\in(\robclf,\mu_+)\,,\\
        \end{cases}
    \end{equation}
    which indicates that $\robclf$ essentially minimizes the
    robust classification error.
    
    Therefore, by solving $\frac{\partial}{\partial \theta}\roberr{\theta}\bigg{\vert}_{\theta=\robclf}=0$,
    we are able to derive the robust classifier as follows.
    \begin{multline*}
        \robclf=\mu_--\frac{\pnDist-(K^2+1)\epsilon}{K^2-1}+\\
        \frac{K}{K^2-1}\sqrt{2(K^2-1)\ln(K)+
        (\pnDist-2\epsilon)^2}\,,
    \end{multline*}
    which satisfies the following
    \begin{equation*}
        \robclf\leq \mu_+-\epsilon\leq\mu_+\,.
    \end{equation*}
    \item The next step is to compare the three different classifiers we just obtained.
    We start with the Bayes and fair classifiers. 
    \begin{align*}
        &\optclf-\fairclf\\
        =~&\frac{K\left(\sqrt{2(K^2-1)\ln(K)+(\pnDist)^2}-\pnDist\right)}{K^2-1}\\
        >~&\frac{K\left(\sqrt{(\pnDist)^2}-\pnDist\right)}{K^2-1}\\
        =~&0\,,
    \end{align*}
    where the inequality comes from the assumption that $K$ is strictly larger than
    $1$. It indicates that the threshold
    of the Bayes classifier is greater than that of the fair classifier. 
    Then, we move on to the comparison between the Bayes
    and robust classifier. Note that the robust classifier
    is identical to the Bayes classifier when $\epsilon$ is $0$, i.e., $\theta_{\mathrm{r}}^{(0)}=\;\optclf$. 
    Besides, we have that
    \begin{align}
         \frac{\partial}{\partial \epsilon}\robclf=~&
        \frac{K^2+1-\frac{2(\pnDist-2\epsilon)K}{\sqrt{(\pnDist-2\epsilon)^2+2(K^2-1)\ln(K)}}}{K^2-1}\nonumber\\
        >~&\frac{K^2+1-\frac{2(\pnDist-2\epsilon)K}{\sqrt{(\pnDist-2\epsilon)^2}}}{K^2-1}\label{eq:prop_4_1_1}\\
        =~&\frac{K^2-2k+1}{K^2-1}\label{eq:prop_4_1_2}\\
        =~&\frac{K-1}{K+1}>0\nonumber\,,
    \end{align}
    where Equation \eqref{eq:prop_4_1_1} is due to the fact that $K>1$ and Equation \eqref{eq:prop_4_1_2} comes
    from the assumption  that $\epsilon\leq \frac{\pnDist}{2}$,
    which ensures that $\pnDist-2\epsilon$ is strictly
    positive.
    Therefore, the
    partial derivative of $\robclf$ in $\epsilon$ is
    strictly positive over the interval $\left[0,\frac{\pnDist}{2}\right]$. As a consequence,
   for any $\epsilon\in\left[0,\frac{\pnDist}{2}\right]$,
    the following relation always holds that $\mu_+\geq\robclf \geq \theta_{\mathrm{r}}^{(0)}=\;\optclf$.
    Putting things together, we end up with following relation:
    for any $\epsilon\in\left[0,\frac{\pnDist}{2}\right]$ and $K\in\left(1, \Kbound\right)$,
    \begin{equation*}
            \mu_-+\epsilon\leq \fairclf\leq\; \optclf\;\leq \robclf\leq\mu_+ -\epsilon\,.
        \end{equation*}
    \end{enumerate}
\end{proof}

{\bf Proof of Corollary 6.2}
\begin{proof}
    Note that, by Equation \eqref{eq:monotonicity}, the robust classification
    error is strictly decreasing over $(\mu_-,\robclf)$.
    Then, by Equation \eqref{eq:hierarchy} in Proposition \ref{prop:three_clfs}, the three classifiers satisfy the
    following relation
    \begin{equation*}
        \mu_-\leq \fairclf\leq \;\optclf\;\leq \robclf\,.
    \end{equation*}
    Due to contiguity of $\roberr{\theta}$, we can argue that,
    for any $\epsilon\in\left(0,\frac{\pnDist}{2}\right)$ and $K\in\left(1, \Kbound\right)$,
    \begin{equation*}
        \roberr{\fairclf}\geq\roberr{\optclf}\geq\roberr{\robclf}\,.
    \end{equation*}
\end{proof}

{\bf Proof of Corollary 6.3}

\begin{proof}
The boundary error and its partial derivative are presented in the following:
    \begin{align*}
     &\bdyerr{\theta}\\
     =~&\pr{\exists~\vert\tau\vert\leq \epsilon,~f_{\theta}(X+\tau)\neq Y, ~f_{\theta}(X)=Y}\\
            =~&\frac{1}{2}\pr{\exists~\vert\tau\vert\leq \epsilon,~f_{\theta}(X+\tau)=-1,~f_{\theta}(X)=1\mid Y=1}+\\
        &\frac{1}{2}\pr{\exists~\vert\tau\vert\leq \epsilon,~f_{\theta}(X+\tau)=1,~f_{\theta}(X)=-1\mid Y=-1}\\
        =~&\frac{1}{2}\pr{\theta<X\leq \theta+\epsilon\mid Y=1}+\\
        &\frac{1}{2}\pr{\theta\geq X>\theta-\epsilon\mid Y=-1}\\
        =~&\frac{1}{2}\int_{\theta}^{\theta+\epsilon}\frac{1}{\sqrt{2\pi}K}\exp\left(-\frac{(x-\mu_+)^2}{2K^2}\right)dx+\\
        &\frac{1}{2}\int_{\theta-\epsilon}^{\theta}\frac{1}{\sqrt{2\pi}}\exp\left(-\frac{(x-\mu_-)^2}{2}\right)dx\,,
    \end{align*}
    and
    \begin{align*}
        &\frac{\partial}{\partial \theta}\bdyerr{\theta}\\
        =~&\frac{\exp\left(-\frac{(\theta+\epsilon-\mu_+)^2}{2K^2}\right)}{2\sqrt{2\pi}K}-\frac{\exp\left(-\frac{(\theta-\mu_+)^2}{2K^2}\right)}{2\sqrt{2\pi}K}+\\
        &\frac{\exp\left(-\frac{(\theta-\mu_-)^2}{2}\right)}{2\sqrt{2\pi}}-
        \frac{\exp\left(-\frac{(\theta-\epsilon-\mu_-)^2}{2}\right)}{2\sqrt{2\pi}}\\
        =~&\frac{g(\epsilon;\theta,K)-g(0;\theta,K)}{2\sqrt{2\pi}}\,,
    \end{align*}
    where $g:\RR\mapsto\RR$ is an auxiliary function shown as follows
    \begin{multline}\label{eq:func_g}
        g(\epsilon;\theta,K)\coloneqq
        \\\frac{\exp\left(-\frac{(\theta+\epsilon-\mu_+)^2}{2K^2}\right)}{K}-\exp\left(-\frac{(\theta-\epsilon-\mu_-)^2}{2}\right)\,.
    \end{multline}
    By Proposition \ref{prop:func_g_monotone}, the partial derivative of the boundary
    error in $\theta$ is always negative for any $\theta\in\left[\fairclf,\robclf\right]$
    because
    \begin{equation*}
        \frac{\partial}{\partial\theta}\bdyerr{\theta}=\frac{g(\epsilon;\theta,K)-g(0;\theta,K)}{2\sqrt{2\pi}}<0\,.
    \end{equation*}
    It leads to the following result that the boundary error $\roberr{\theta}$ decreases in $\theta$ over $\left[\fairclf,\robclf\right]$, and, therefore,
    \begin{equation*}
        \bdyerr{\fairclf}\geq\bdyerr{\optclf}\geq
        \bdyerr{\robclf}\,.
    \end{equation*}
\end{proof}

\begin{proposition}\label{prop:func_g_monotone}
        For any $\epsilon\in\left[0,\frac{\pnDist}{4}\right]$, $K\in\left(1, \newKbound\right)$, and $\theta\in\left[\fairclf, \robclf\right]$, the following relation holds
        \begin{equation}\label{eq:func_g_monotone}
            g(\epsilon;\theta,K)\leq g(0;\theta,K)\,,
        \end{equation}
        where the function $g$ is defined in Equation \eqref{eq:func_g}.
\end{proposition}

\begin{proof}
    Note that the partial derivative of the function $g$ in $\epsilon$ can be given by 
    \begin{align*}
        \frac{\partial}{\partial\epsilon} g(\epsilon;\theta,K)=~&\frac{\mu_+-\theta-\epsilon}{K^3}\exp\left(-\frac{(\mu_+-\theta-\epsilon)^2}{2K^2}\right)-\\
        &(\theta-\epsilon-\mu_-)\exp\left(-\frac{(\theta-\epsilon-\mu_-)^2}{2}\right)\,.
    \end{align*}
    In order to establish the result in Equation \eqref{eq:func_g_monotone}, it suffices to demonstrate that the partial derivative $\frac{\partial}{\partial\epsilon} g(\epsilon;\theta,K)$ is negative, which implies that the function $g(\epsilon;\theta,K)$ is strictly decreasing
    over $\left[0,\frac{\pnDist}{4}\right]$.
    Note that
    \begin{align*}
        &\ln\left(\frac{\mu_+-\theta-\epsilon}{K^3}\exp\left(-\frac{(\mu_+-\theta-\epsilon)^2}{2K^2}\right)\right)-\\
        &\ln\left((\theta-\epsilon-\mu_-)\exp\left(-\frac{(\theta-\epsilon-\mu_-)^2}{2}\right)\right)\\
        =~&-\left[\ln(\theta-\epsilon-\mu_-)-\ln(\mu_+-\theta-\epsilon)-\right]+\\
        &\left[\frac{(\theta-\epsilon-\mu_-)^2}{2}-\frac{(\mu_+-\theta-\epsilon)^2}{2K^2}\right]-3\ln(K)\\
        =~&\frac{1}{2}\left(q(\theta;\epsilon,K)-p(\theta;\epsilon,K)-4\ln(K)\right)\\
        \leq~&0\,,
    \end{align*}
    where the last inequality comes from Proposition \ref{prop:main_aux}.
    It follows that, due to monotonicity of the function $x\mapsto \ln(x)$,
    \begin{align*}
        \frac{\partial}{\partial\epsilon} g(\epsilon;\theta,K)=~&\frac{\mu_+-\theta-\epsilon}{K^3}\exp\left(-\frac{(\mu_+-\theta-\epsilon)^2}{2K^2}\right)-\\
        &(\theta-\epsilon-\mu_-)\exp\left(-\frac{(\theta-\epsilon-\mu_-)^2}{2}\right)\\
        \leq~&0\,,
    \end{align*}
    which completes our proof here.
\end{proof}

\begin{proposition}\label{prop:main_aux}
 For any $\epsilon\in\left[0,\frac{\pnDist}{4}\right]$, $K\in\left(1, \newKbound\right]$, and $\theta\in\left[\fairclf, \robclf\right]$, the following relation always holds:
    \begin{equation*}
        p(\theta;\epsilon,K)\geq q(\theta;\epsilon,K)-4\ln(K)\,,
    \end{equation*}
    where
    \begin{align}
        p(\theta;\epsilon,K)=~&\ln\left((\theta-\epsilon-\mu_-)^2\right)-
        \ln\left(\frac{(\mu_+-\theta-\epsilon)^2}{K^2}\right)\,,\label{eq:func_p}\\
        q(\theta;\epsilon,K)=~&(\theta-\epsilon-\mu_-)^2-
        \frac{(\mu_+-\theta-\epsilon)^2}{K^2}\,.\label{eq:func_q}
    \end{align}
\end{proposition}

\begin{proof}
    First off, observe that the functions $p$ and $q$ are both increasing over
    $\left[\fairclf, \robclf\right]$.
    It follows that, for any $\left[\fairclf, \robclf\right]$,
    \begin{align}
        &p(\theta;\epsilon,K)\geq ~p(\fairclf;\epsilon,K)\label{eq:mono_1}\\
        \geq~&-2\ln(K)=2\ln(K)-4\ln(K)\label{eq:aux_1}\\
        =~& q\left(\robclf;\epsilon,K\right)-4\ln(K)\label{eq:aux_2}\\
        \geq~&q(\theta;\epsilon,K)-4\ln(K)\label{eq:mono_2}\,,
    \end{align}
    where Equation \eqref{eq:mono_1} and \eqref{eq:mono_2}
    come from monotonicity of the functions $p$ and $q$. Equation
    \eqref{eq:aux_1} and \eqref{eq:aux_2} are due to Proposition \ref{eq:prop_aux_2} and \ref{eq:prop_aux_1} respectively.
\end{proof}

\begin{proposition}\label{eq:prop_aux_1}
For any $\epsilon\in\left[0,\frac{\pnDist}{4}\right]$ and $K\in\left(1, \newKbound\right]$,
    \begin{equation*}
        q\left(\robclf;\epsilon,K\right)=2\ln(K)\,,
    \end{equation*}
    where the definition of the function $q$ is given in Equation \eqref{eq:func_q}.
\end{proposition}

\begin{proof}
    Due to optimality of $\robclf$ in terms of robust error, $\robclf$ should 
    be a solution to Equation \eqref{eq:rob_opt}, i.e.,
    \begin{multline*}
        \exp\left(-\frac{(\robclf+\epsilon-\mu_+)^2}{2K^2}\right)-\\K\exp\left(-\frac{(\robclf-\epsilon-\mu_-)^2}{2}\right)=0\,,
    \end{multline*}
    which leads to the following, by multiplying the both sides with the
    term $\exp\left(\frac{(\robclf-\epsilon-\mu_-)^2}{2}\right)$,
    \begin{equation*}
        \exp\left(\frac{q(\robclf;\epsilon,K)}{2}\right)=K\,,
    \end{equation*}
    and
    \begin{equation*}
        q(\robclf;\epsilon,K)=2\ln(K)\,.
    \end{equation*}
\end{proof}

\begin{proposition}\label{eq:prop_aux_2}
    For any $\epsilon\in\left[0,\frac{\pnDist}{4}\right]$ and $K\in\left(1, \newKbound\right]$,
    \begin{equation}\label{eq:prop_57}
        p\left(\fairclf;\epsilon,K\right)\geq-2\ln(K)\,,
    \end{equation}
    where the definition of the function $p$ is given in Equation \eqref{eq:func_p}.
\end{proposition}

\begin{proof}
    Equation \eqref{eq:prop_57} can be rewritten into the following
    equivalent form
    \begin{equation}\label{eq:equiv_form_prop_57}
        \exp\left(p\left(\fairclf;\epsilon,K\right)+2\ln(K)\right)=K-\frac{(K-1)\epsilon}{\frac{\pnDist}{K+1}-\frac{\epsilon}{K}}>1\,.
    \end{equation}
    Note that
    \begin{align*}
        K\leq\newKbound&\implies K+\frac{1}{K}\leq \frac{\pnDist}{\epsilon}-2\\
        &\implies \frac{\pnDist}{K+1}\geq \frac{K+1}{K}\epsilon\,,
    \end{align*}
    which leads to the following result
    \begin{align*}
        &\exp\left(p\left(\fairclf;\epsilon,K\right)+2\ln(K)\right)\\
        =~&K-\frac{(K-1)\epsilon}{\frac{\pnDist}{K+1}-\frac{\epsilon}{K}}\geq K-(K-1)\\
        =~&1\,.
    \end{align*}
    It helps complete our proof here.
\end{proof}

{\bf Proof of Theorem 6.4}
\begin{proof}
    Since $f_{\theta}$ is essentially a linear classifier, the distance to the decision boundary of $f_{\theta}$ is simply the absolute value between the feature $X$ and the threshold $\theta$, i.e., $\dist{X}{\theta}=\vert X-\theta\vert$. 
    Notice that the average distance to the decision boundary
    of $f_{\theta}$ can then be given by
    \begin{align*}
        \expect{\dist{X}{\theta}}=~&\expect{\vert X-\theta\vert}\\
        =~&\expect{\vert X-\theta\vert\mid Y=1}\cdot \pr{Y=1}+\\
        &\expect{\vert X-\theta\vert\mid Y=-1}\cdot \pr{Y=-1}\\
        =~&\frac{1}{2}\int_{-\infty}^{+\infty}\frac{\vert x-\theta\vert}{\sqrt{2\pi}K}\exp\left(-\frac{(x-\mu_+)^2}{2K^2}\right)dx+\\
        &\frac{1}{2}\int_{-\infty}^{+\infty}\frac{\vert x-\theta\vert}{\sqrt{2\pi}}\exp\left(-\frac{(x-\mu_-)^2}{2}\right)dx\,,
    \end{align*}
    whose derivative can be expressed as
    \begin{equation*}
        \left(\expect{\dist{X}{\theta}}\right)'=2\left(\Phi(\theta-\mu_-)-\Phi\left(\frac{\mu_+-\theta}{K}\right)\right)\,,
    \end{equation*}
    where $\Phi$ represents the cumulative distribution function
    associated with the standard normal distribution.
    Recall that 
    \begin{equation*}
        \fairclf=\mu_-+\frac{\pnDist}{K+1}\,.
    \end{equation*}
    It follows that
    \begin{equation*}
        \left(\expect{\dist{X}{\theta}}\right)'\begin{cases}
        <0\,, & \theta\in (\mu_-,\fairclf)\,,\\
        >0\,, &\theta\in(\fairclf, \mu_+)\,,
        \end{cases}
    \end{equation*}
    which implies that the average distance strictly decreases
    over $(\mu_-, \fairclf)$ while increases over
    $(\fairclf,\mu_+)$.
    By the relation shown in Equation \eqref{eq:hierarchy},
    we figure out that, for any $\epsilon\in\left[0,\frac{\pnDist}{2}\right]$ and $K\in\left(1, \Kbound\right]$,
    \begin{equation*}
         \expect{\dist{X}{\robclf}}\geq\expect{\dist{X}{\optclf}}\geq\expect{\dist{X}{\fairclf}}\,.
    \end{equation*}
    Moreover, $\fairclf$ is the minimizer of the average distance $\expect{\dist{X}{\theta}}$ over the interval
    $[\mu_-,\mu_+]$.
\end{proof}

{\bf Proof of Theorem 6.5}

\begin{proof}
First off, notice that $\fairclf(\lambda)$ is the minimizer of
$\naterr{\theta}+\lambda\cdot \cL^{\mathrm{fair}}_{\theta}$ over $[\mu_-,\mu_+]$, where $ \cL^{\mathrm{fair}}_{\theta}$ is 
a shorthand for
\begin{equation*}
    \left\vert \pr{f_{\theta}(X)\neq 1\mid Y=1}
    -\pr{f_{\theta}(X)\neq -1\mid Y=-1}\right\vert\,.
\end{equation*}
Thus, $\naterr{\theta}+\lambda\cdot \cL^{\mathrm{fair}}_{\theta}$
can be presented in a piecewise way as follows:
\begin{enumerate}[1)]
    \item if $\theta\leq \fairclf$,
    \begin{multline*}
        \naterr{\theta}+\lambda\cdot \cL^{\mathrm{fair}}_{\theta}=
         (1-\lambda)\pr{f_{\theta}(X)\neq 1\mid Y=1}+\\
    (1+\lambda)\pr{f_{\theta}(X)\neq -1\mid Y=-1}\,;
    \end{multline*}
    \item if $\theta>\fairclf$,
    \begin{multline*}
        \naterr{\theta}+\lambda\cdot \cL^{\mathrm{fair}}_{\theta}=
         (1+\lambda)\pr{f_{\theta}(X)\neq 1\mid Y=1}+\\
    (1-\lambda)\pr{f_{\theta}(X)\neq -1\mid Y=-1}\,.
    \end{multline*}
\end{enumerate}
We start with the first case where $\theta$ is no greater than $\fairclf$. The partial derivative of $ \naterr{\theta}+\lambda\cdot \cL^{\mathrm{fair}}_{\theta}$
is then given by
\begin{multline}\label{eq:left_case_aux}
    \frac{\partial \left( \naterr{\theta}+\lambda\cdot \cL^{\mathrm{fair}}_{\theta}\right)}{\partial\theta} 
    =\frac{1-\lambda}{2\sqrt{2\pi}K}\exp\left(-\frac{(\theta-\mu_+)^2}{2K^2}\right)-\\\frac{1+\lambda}{2\sqrt{2\pi}}\exp\left(-\frac{(\theta-\mu_-)^2}{2}\right)\,.
\end{multline}
Observe that, when $\lambda\in[0,1]$, this partial derivative is increasing over $[\mu_-, \fairclf]$ 
with its value at $\fairclf$ no greater than $0$ while it is always non-positive for any $\lambda\in(1,+\infty)$ and
$\theta\in[\mu_-,\fairclf]$. Therefore, the partial derivative 
presented in Equation \eqref{eq:left_case_aux} proves to be
non-positive, whatever $\lambda$, which implies that the
function $\naterr{\theta}+\lambda\cdot \cL^{\mathrm{fair}}_{\theta}$ decreases over $[\theta,\fairclf]$
and we merely need to focus on the case of $\theta\in[\fairclf,\mu_+]$ in pursuit of its minimizer.
Then, we continue to investigate the case where $\theta$ is greater
than $\fairclf$. Likewise, the partial derivative of $ \naterr{\theta}+\lambda\cdot \cL^{\mathrm{fair}}_{\theta}$ can be expressed as
\begin{multline}\label{eq:right_case_aux}
    \frac{\partial \left( \naterr{\theta}+\lambda\cdot \cL^{\mathrm{fair}}_{\theta}\right)}{\partial\theta} 
    =\frac{1+\lambda}{2\sqrt{2\pi}K}\exp\left(-\frac{(\theta-\mu_+)^2}{2K^2}\right)-\\\frac{1-\lambda}{2\sqrt{2\pi}}\exp\left(-\frac{(\theta-\mu_-)^2}{2}\right)\,.
\end{multline}
We split our studies into the following three scenarios:
\begin{enumerate}[1)]
    \item if $\lambda\in \left[0,\frac{K-1}{K+1}\right]$,
    $ \naterr{\theta}+\lambda\cdot \cL^{\mathrm{fair}}_{\theta}$
    first decreases over $[\fairclf,\fairclf(\lambda)]$ and then
    increases over $[\fairclf(\lambda), \mu_+]$ where its minimum takes place at
    \begin{multline}\label{eq:fair_penalty_sol}
        \mu_--\frac{\pnDist}{K^2-1}+\\
            \frac{K}{K^2-1}\sqrt{2(K^2-1)\ln\left(\frac{1-\lambda}{1+\lambda}\cdot 
     K\right)+(\pnDist)^2}\,;
    \end{multline}
    \item if $\lambda\in \left(\frac{K-1}{K+1},1\right]$, the partial derivative
    in Equation \eqref{eq:right_case_aux} turns out to be increasing
    in $\theta$ with its value at $\fairclf$ non-negative.
    It implies that the partial derivative is always non-negative and, thus, the function $\naterr{\theta}+\lambda\cdot \cL^{\mathrm{fair}}_{\theta}$ increases over $[\fairclf, \mu_+]$. As a consequence, the minimizer $\fairclf(\lambda)$ actually coincides with $\fairclf$.
    \item if $\lambda\in(1,+\infty)$, note that the partial derivative is always positive. Following the same reasoning in the previous scenario, we figure out that
    the minimizer $\fairclf(\lambda)$ is identical to $\fairclf$.
\end{enumerate}
Since the function $\lambda\mapsto \frac{1-\lambda}{1+\lambda}$ is decreasing over $\left[0, \frac{K-1}{K+1}\right]$, by Equation \eqref{eq:fair_penalty_sol}, we can argue that 
$\fairclf(\lambda)$ is decreasing in $\lambda$ over $\RR_+$ as well. Furthermore, by the proof of Theorem \ref{thm:avg_dist},
the average distance to the decision boundary is an increasing function in $\theta$ over $[\fairclf, \mu_+]$, which indicates that
the average distance associated with the classifier $f_{\fairclf(\lambda)}$ decreases, as $\lambda$ increases.
\end{proof}

\section{Fairness and Robustness Models} 
\label{app:fairness_defs}

\subsection{Fair models}
\label{app:fair_models}
The main text mainly discussed penalty-based methods as a way to encourage accuracy parity in a classifier. This section discusses a second methodology to achieve fairness denoted \emph{group-loss focused} methods \cite{tian_li}. We will show that the main conclusion of this paper (e.g., that fairness increase adversarial vulnerability) holds regardless of the methodology adopted to achieve fairness. 

\noindent\textbf{Group-loss focused methods.} 
Methods in this category force the training to focus on the loss component of worst performing groups. An effective method to achieve this goal was proposed in \cite{tian_li}:
\begin{align}
    \fairclf &= \underset{\theta}{\argmin}~ \sum_{a \in \cA}\frac{1}{q+1}\cL_{\theta}(D_a)^{q+1},
    \label{eqn:q-learn}
\end{align}
where $q$ is a non-negative constant. The intuition behind powering the loss by positive number $q+1$ is to penalize more the classes that have the larger losses. Thus, $q$ plays the role of the fairness parameter, like $\lambda$ in penalty-based methods: larger $q$ or $\lambda$ values are associated with fairer (but also often less accurate) models.  The main differences between penalty-based methods and group-loss focused methods are the following. First, the loss function of group-loss focused methods is fully differentiable, in contrast to that of penalty-based methods, which is sub-differentiable when the group loss equals the population loss. Second, penalty-based methods try to equalize the losses across various subgroup, while group-focused based methods attempt at minimizing the maximum loss across all subgroups.

% {\color{blue}{@Cuong: Can you describe the two attacks adopted in a new subsection? Just a brief description suffices.}}

\section{Datasets and Settings}
\label{app:datasets}

\paragraph{Datasets.} Experiments were performed using three benchmark datasets: UTK Face \cite{zhifei2017cvpr}, CIFAR-10 \cite{krizhevsky2009learning}, and Fashion MNIST (FMNIST) \cite{xiao2017fashion}.

\begin{enumerate}
    \item UTK Face \cite{zhifei2017cvpr}. It consists of more than 20,000 facial images of 48x48 pixels resolution. The experiments consider two learning tasks: 
    (1) The first splits the data into five ethnicities:  White, Black, Asian, Indian, and Others.  
    (2) The second splits the data into nine age bins: under-ten years old, 10-14, 15-19, 20-24, 25-29, 30-39, 40-49, 50-59, and over 59 years old.
    The classes are not uniformly distributed per number of groups and do not contain the same number of images in each group. 
    An 80/20 train-test split is performed.

    \item CIFAR-10 \cite{krizhevsky2009learning}. It consists of 60,000 32x32 coloured images belonging to  10 classes, with 6000 images per class. The training set has 50,000 images while the test set has 10,000 images.
    
    \item Fashion MNIST(FMNIST) \cite{xiao2017fashion}. It consists of 60,000 28x28 gray-colored images belonging to 10 classes, with 6000 images per class. The training set has 50,000 images while the test set has 10,000 images.
\end{enumerate}

\paragraph{Network architectures.} The experiments consider three network architectures of increasing complexity:
\begin{enumerate}
    \item A CNN consisting of 3 convolutions layers followed by 3 fully connected layers. 
    \item A ResNet 50 network \cite{he2016deep} (over 23 million parameteres).
    \item A VGG-13 network \cite{simonyan2014very} (138 million parameters).
\end{enumerate}

We use CNN to provide some examples in the main text, but in this Appendix we will focus on the ResNet and VGG networks.  
\paragraph{Fair models and parameter settings.}
For penalty-based methods the experiments vary the range of fairness parameter $\lambda \in [0,2]$. We note that larger $\lambda$ values may have a detrimental effects to the models accuracy, which, we believe, limit the applicability of the fairness methods in real use cases, thus we focus on these more realistic scenarios. 
However, we also note that appropriate choices of hyperparameters $\lambda$ will necessarily depend on the task and architecture at hand and can be used to balance the trade-off between accuracy and fairness. 

For group-loss focused methods, the experiments consider $q \in [0,2]$ for similar reasons as those stated above.

The set of hyper-parameters, learning rate (lr), batch-size (bs), and number of training epochs (epochs) adopted, for each dataset and architecture is reported in Table \ref{tab:setting}. 

\begin{table}[h!]
\centering
\begin{tabular}{c| c c c} 
 \toprule
 Dataset & lr & bs & epochs \\ [0.5ex] 
 \midrule
 UTK Face& 1e-3 & 32 & 70 \\ 
 \hline
 CIFAR-10 & 1e-2 & 32 & 200 \\
 \hline
 FMNIST & 1e-2 & 32 & 50 \\
 \bottomrule
\end{tabular}
\caption{Hyperparameters settings for each datset.}
\label{tab:setting}
\end{table}

For each setting, the experiments report the average results of 10 runs, each initializing the models parameters using a different random seed.

\paragraph{Adversarial attacks.} The experiments also consider two classes of adversarial attacks to test the model robustness: (1) The $l_{\infty}$ RFGSM attacks \cite{tramer2017ensemble} and (2) The $l_2$ PGD attacks \cite{madry2017towards}. The experiments adopt the implementations reported in the  Python package \emph{torchattacks} \cite{kim2020torchattacks}.

\paragraph{Code and Preprocessing.} Follow standard setting, the range of the pixel values was normalized in $[0,1]$ for all datasets adopted.

All codes were written in  Python 3.7 and in Pytorch 1.5.0. %All fair models were implemented by ourselves. 
The library \emph{torchattacks} \cite{kim2020torchattacks} was adopted to generate different adversarial attacks. The repository contained the dataset and implementation will be released publicly upon paper acceptance.

\section{Additional experiments}
\label{app:experiments}

This section describes additional experiments to further support the claims reported in the main paper. In particular, the experiments report results for deeper networks (e.g., ResNet-50) and for additional fairness methods and adversarial attacks.

\subsection{Fairness impacts on the decision boundary}
\label{sec:app_exp_db}
Recall, from Theorem \ref{thm:6.5}, that fairness reduces the average distance of the testing samples to the decision boundary. As a consequence the fair classifiers are more vulnerable against the adversarial attacks than the natural (unfair) classifiers. This section provides additional evidence to support this claim on a high-dimensional model (ResNet 50) and using both penalty based methods additional and group-loss focused methods (see Section \ref{app:fair_models}) to derive a fair classifier.

\paragraph{Penalty-based methods.}
Figure \ref{fig:resnet_dist_bndry_reg} and Figure \ref{fig:vgg_dist_bndry_reg} summarize the results obtained by a  penalty-based fair model executed on different benchmark datasets. 
The experiments again report a consistent trend: as more fairness is enforced (increasing the values $\lambda$ ), 
the natural errors (left plots) generally increase while both the fairness violations (middle plots) and the average distance to the decision boundary (right plots) decrease. Recall that the latter is a proxy for measuring the model robustness: the closer are samples to the decision boundary the less robust is a model. Thus the previous plots show that robustness decreases as fairness increase.

\begin{figure}[tb]
%\captionsetup{justification=centering}
\centering
\begin{subfigure}[b]{0.45\textwidth}
\includegraphics[width = 1.0\linewidth]{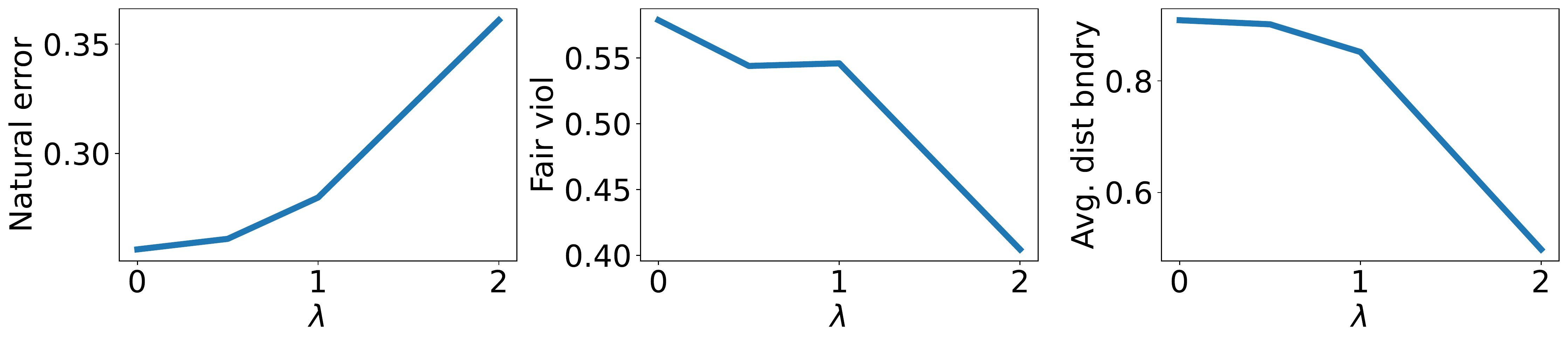}
\caption{UTKFace ethnicity dataset}
\end{subfigure}
\begin{subfigure}[b]{0.45\textwidth}
\includegraphics[width = 1.0\linewidth]{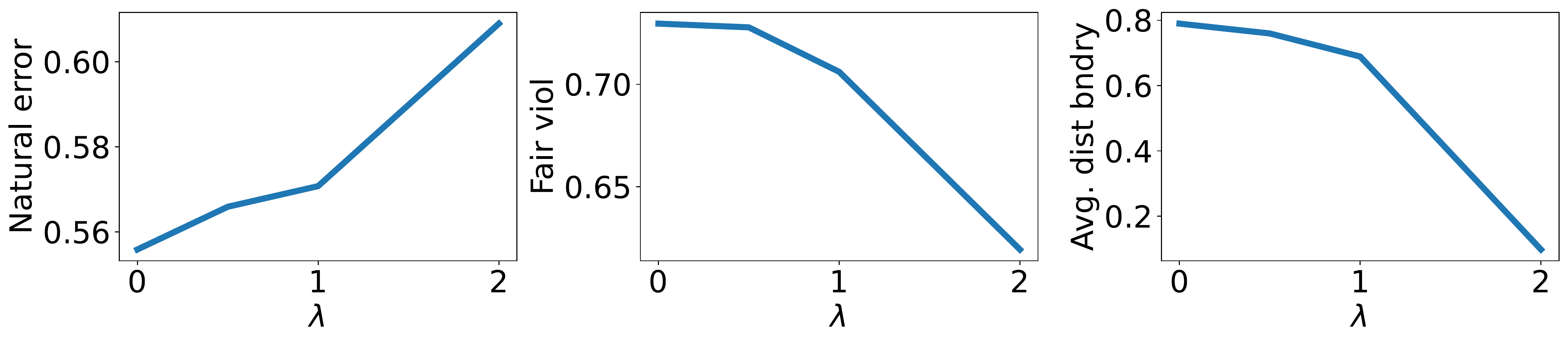}
\caption{UTKFace age bins dataset}
\end{subfigure}
%\begin{subfigure}[b]{0.45\textwidth}
% \includegraphics[width = 1.0\linewidth]{dist_bndry_fair_clf_CIFAR_ResNet50.pdf}
% \caption{CIFAR-10 dataset}
%\end{subfigure}
\begin{subfigure}[b]{0.45\textwidth}
\includegraphics[width = 1.0\linewidth]{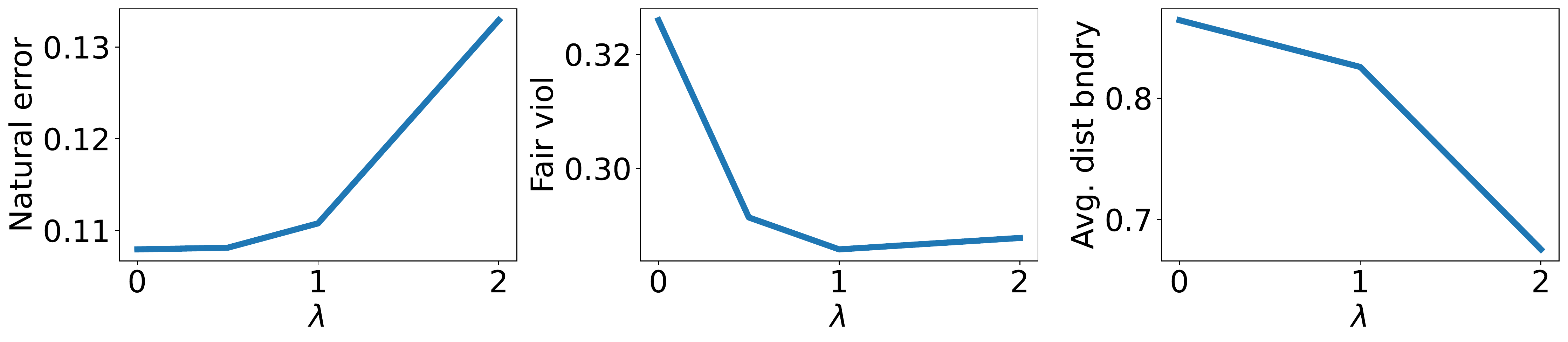}
\caption{FMNIST dataset}
\end{subfigure}
\caption{Natural errors (left), fairness violations (middle), and average distance to the decision boundary (right) for different datasets when varying the fairness parameter $\lambda$ of \textbf{penalty-based  methods} on ResNet-50 networks}
\label{fig:resnet_dist_bndry_reg}
\end{figure}

\begin{figure}[tb]
%\captionsetup{justification=centering}
\centering
\begin{subfigure}[b]{0.45\textwidth}
\includegraphics[width = 1.0\linewidth]{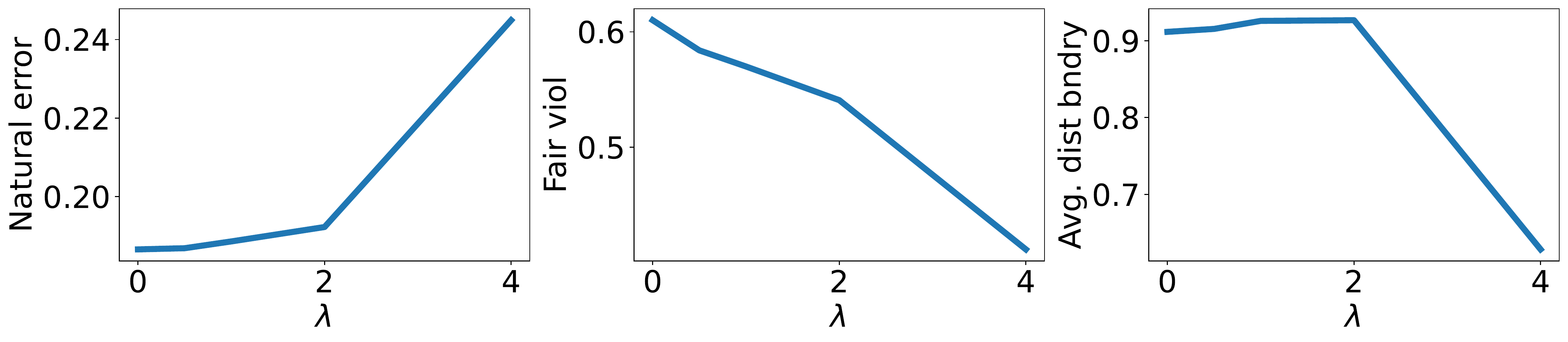}
\caption{UTKFace ethnicity dataset}
\end{subfigure}
\begin{subfigure}[b]{0.45\textwidth}
\includegraphics[width = 1.0\linewidth]{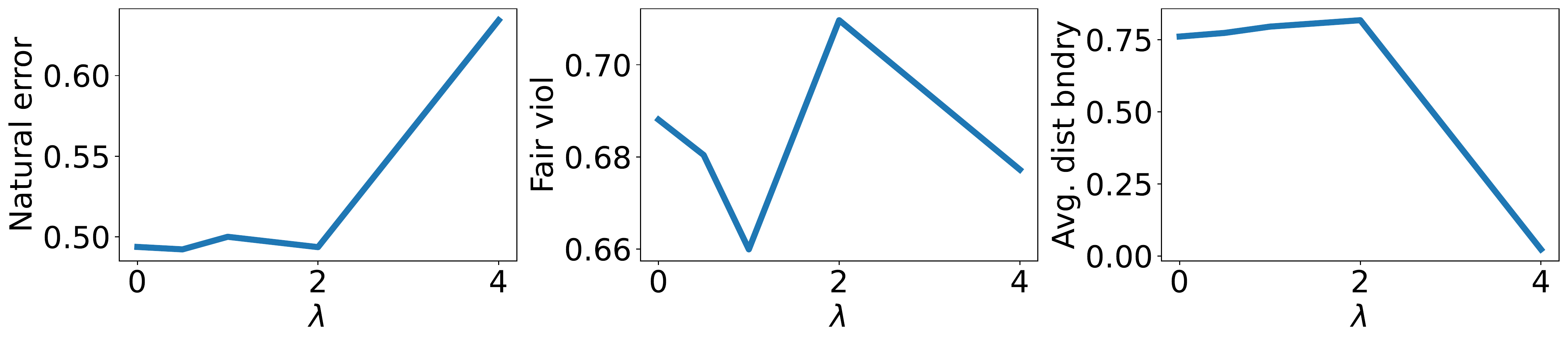}
\caption{UTKFace age bins dataset}
\end{subfigure}
%\begin{subfigure}[b]{0.45\textwidth}
% \includegraphics[width = 1.0\linewidth]{dist_bndry_fair_clf_CIFAR_ResNet50.pdf}
% \caption{CIFAR-10 dataset}
%\end{subfigure}
\caption{Natural errors (left), fairness violations (middle), and average distance to the decision boundary (right) for different datasets when varying the fairness parameter $\lambda$ of \textbf{penalty-based  methods} on VGG 13 networks}
\label{fig:vgg_dist_bndry_reg}
\end{figure}

\paragraph{Group-loss focused methods}
A similar setting is reported for a model satisfying fairness using the group-loss focused method described in section \ref{app:fair_models}. 
This method maximizes the worst group accuracy, which, in turn, attempts at equalizing the accuracy across groups. 
Figure \ref{fig:q_learn_dist_bndry} reports the results obtained using VGG-13 and the UTK-Face dataset. The results again illustrate similar trends: as the fairness parameter $q$ increases, the natural errors (right) tend to increase while both the fairness violations (middle) and the average distance to the decision boundary (left) decrease. 
Notice that small enough $q$-values may also act as a regularizer and have a beneficial effect toward the natural error, as observed for the UTK-age bin task (bottom-left plot). 

% Figure \ref{fig:resnet_dist_bndry_reg} summarizes the results obtained by running penalty-based fair models on different benchmark datasets. We again see a consistently clear trend that : as more fairness is enforced (by using large $\lambda$), the natural errors (left plots) while can increase, the fairness violation (middle plots) decrease, and the average distance to the decision boundary (right plots) reduces accordingly.

\begin{figure}[t]
 \centering
\includegraphics[width=\linewidth]{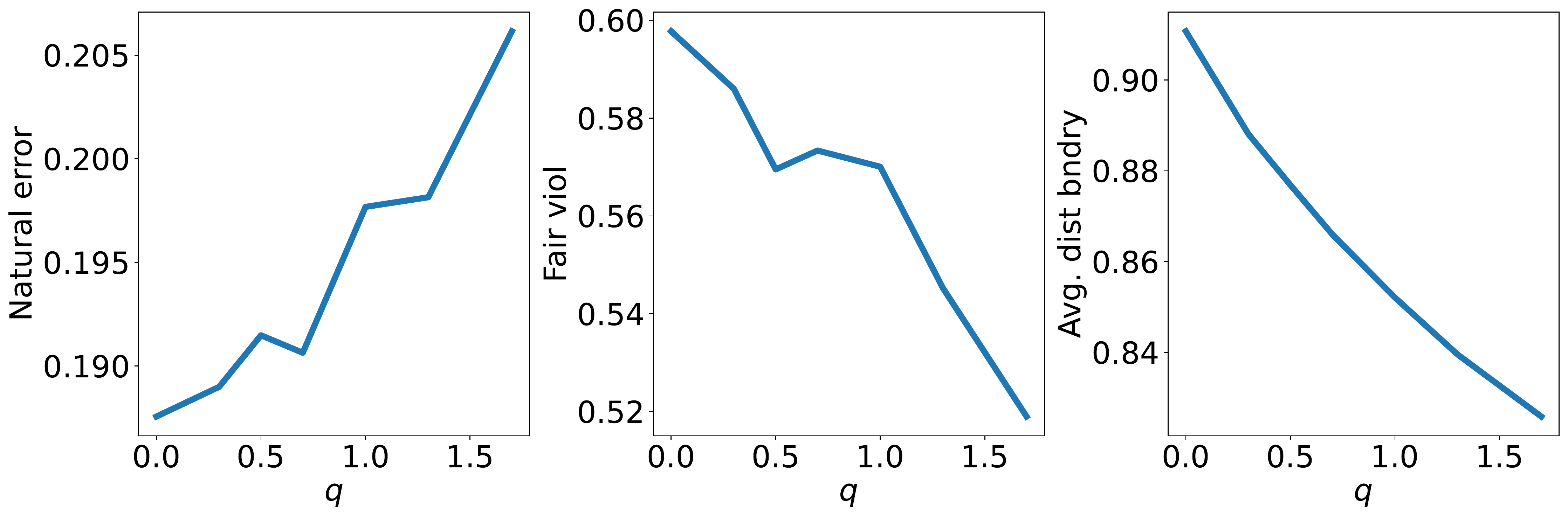}
\includegraphics[width=\linewidth]{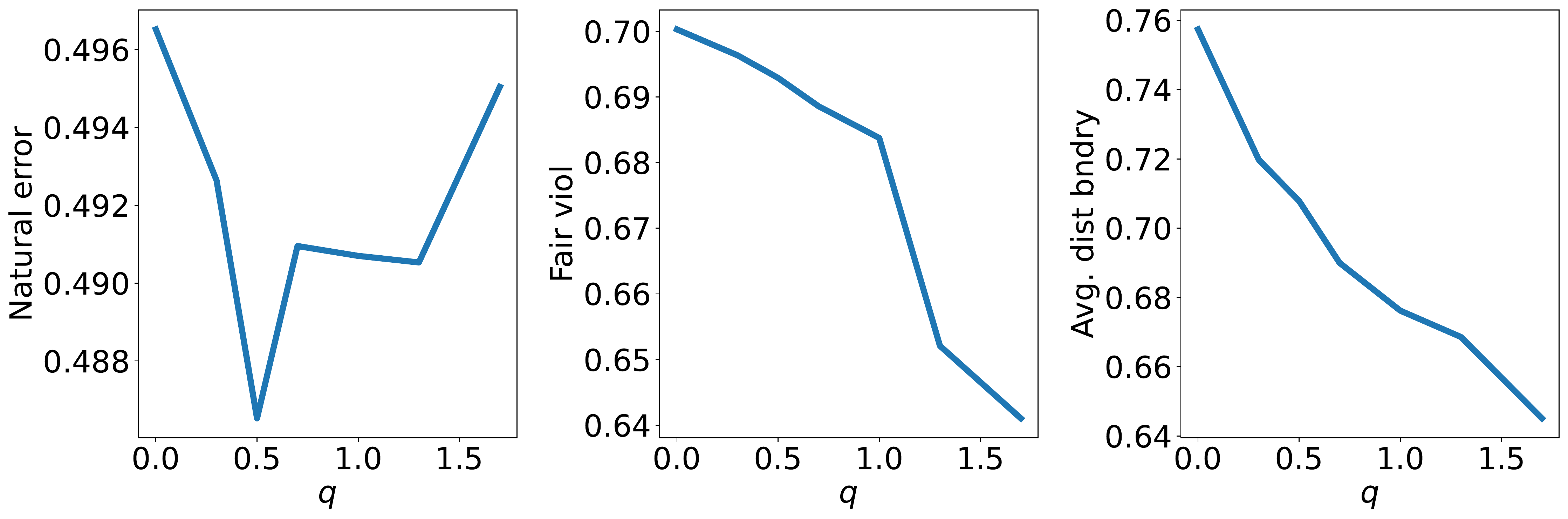}
\caption{Natural errors, fairness violations, and average distance to the decision boundary for the UTK-Face {\sl ethnicity} (top), 
UTK-Face {\sl age bins} (bottom) datasets when varying the fairness parameter $q$ of \textbf{group-loss focused methods} on VGG-13 networks. 
}
\label{fig:q_learn_dist_bndry}
\end{figure}

\paragraph{}

\subsection{Boundary errors increase as fairness decreases}

% \textcolor{red}{we only have group-focused method on this section since there is no mitigation solution for this fair solver.}
This section provides additional experiments to illustrating that the result of the paper ( \emph{fairness increases the adversarial vulnerability}) is invarant across different fair classifier implementations. The experiments adopt  the group-loss focused methods for different values of the fairness parameter $q$. Notice that a natural classifier is obtained when $q=0$. 
Figure \ref{fig:q_learn_rfgsm} displays the natural (left) and boundary (center) errors attained under different level of RFGSM attacks (regulated by parameter $\epsilon$) and the fairness violation (right) on CIFAR (top) FMNIST (middle) and UTK (bottom) datasets. Notice how increasing $q$ to large enough values typically decreases the fairness violations. However, this comes at the cost of increasing the natural error and the boundary errors, which, in turn, exacerbate the robust errors. 

% {\color{blue}{@Cuong: please complete this section}}

\begin{figure}[tb]
%\captionsetup{justification=centering}
\centering
\begin{subfigure}[b]{0.45\textwidth}
\includegraphics[width = 1.0\linewidth]{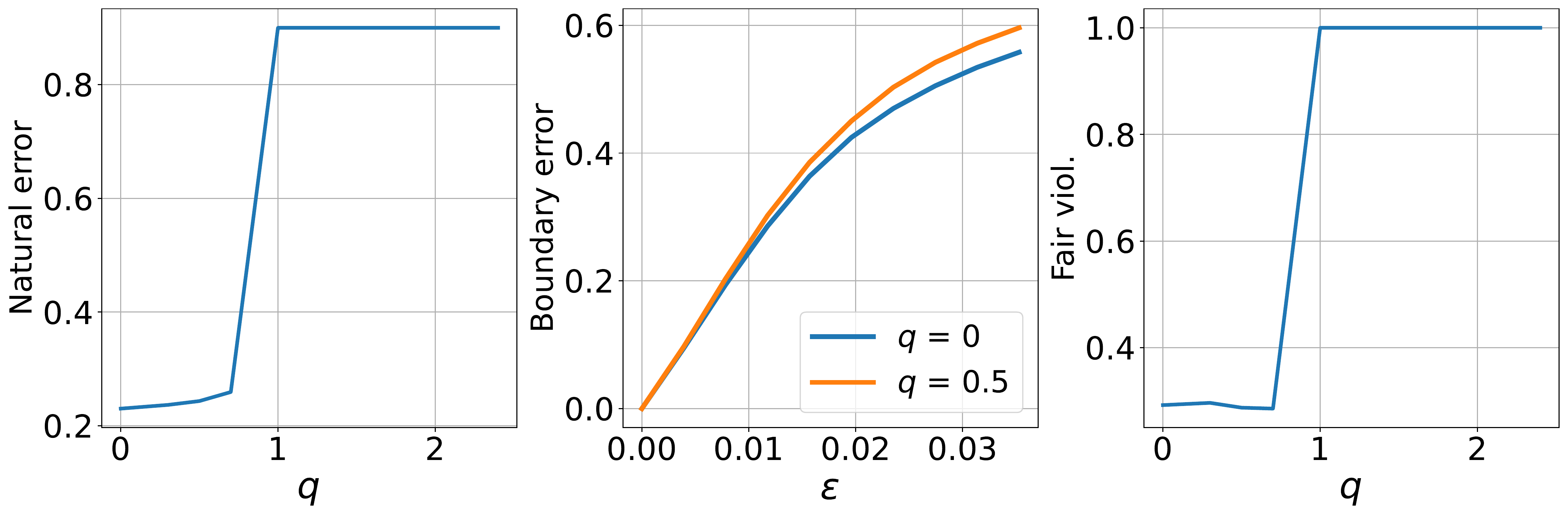}
\caption{CIFAR dataset}
\end{subfigure}
\begin{subfigure}[b]{0.45\textwidth}
\includegraphics[width = 1.0\linewidth]{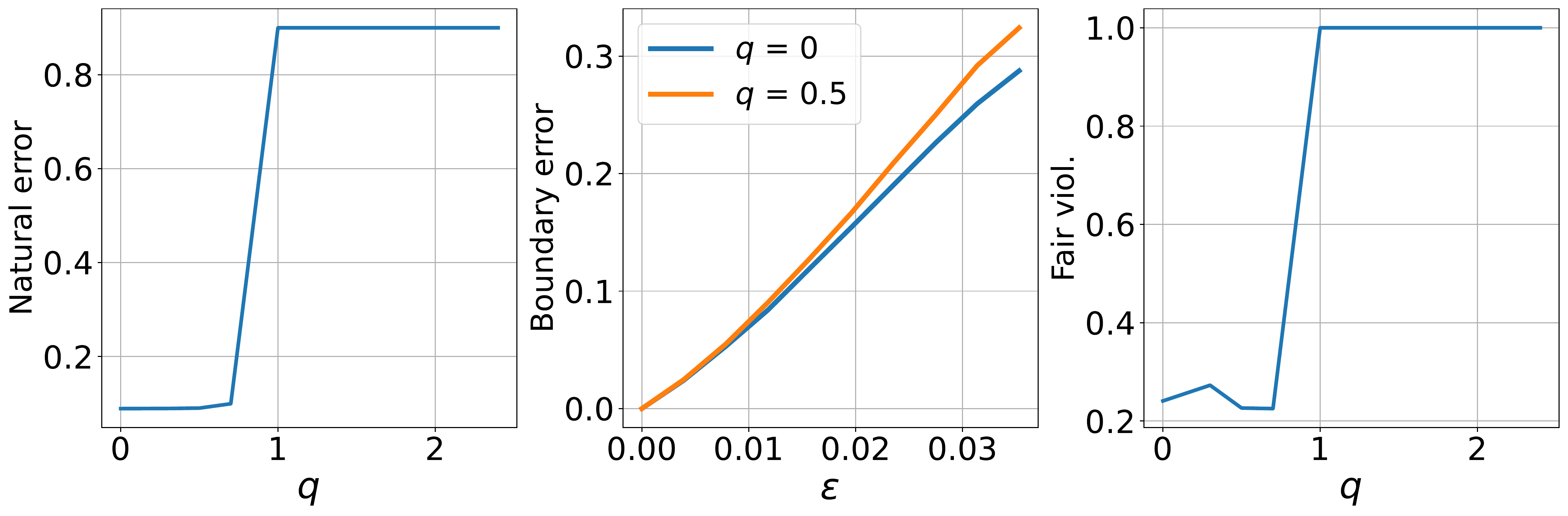}
\caption{FMNIST dataset}
\end{subfigure}
\begin{subfigure}[b]{0.45\textwidth}
\includegraphics[width = 1.0\linewidth]{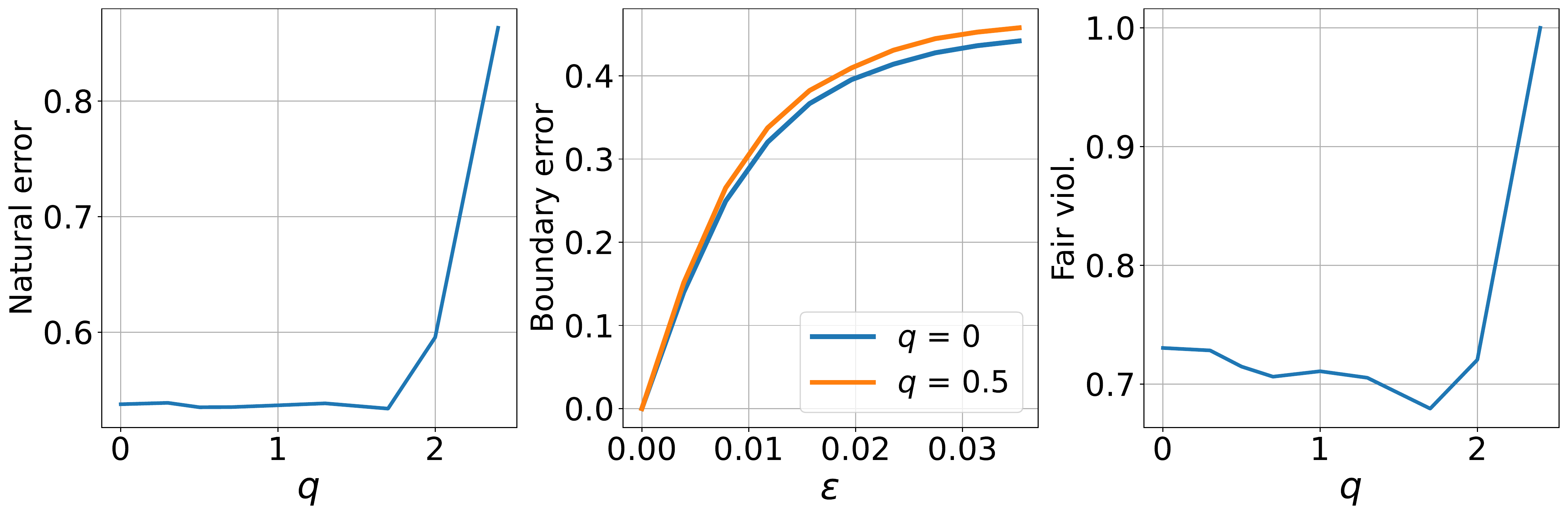}
\caption{UTK age bins dataset}
\end{subfigure}
\caption{Natural errors (left), boundary error  under different RFGSM attacks (middle), and fairness violation (right) of \textbf{group-focused   methods} on ResNet-50 networks}
\label{fig:q_learn_rfgsm}
\end{figure}

\subsection{A Mitigating Solution with Bounded Losses}

\paragraph{More intuition why bounded loss works.}
This section first provides more intuition about why bounded loss functions, such as the ramp loss adopted in the main text, can help reducing the impacts of fairness towards robustness. 
To guide the intuitions, we will refer to Figure \ref{fig:ill_loss}, which plots the graph functions of the following loss function for binary classification tasks:
\begin{itemize}
    \item Ramp loss \cite{goh2016satisfying,collobert2006trading}, defined by: 
    \[\ell(f_{\theta}(X), Y) = \min( 1, \max(0, 1- Y f_{\theta}(X))).\] 
    
    \item Log-loss \cite{hastie2009elements}, defined by: 
    \[\ell(f_{\theta}(X), Y) = \log(1+ \exp(-Y f_{\theta}(X))).\]
    
    \item Exponential loss \cite{hastie2009elements}, defined by: 
    \[\ell(f_{\theta}(X), Y) =  \exp(-Y f_{\theta}(X))).\]
\end{itemize}

\begin{figure}[t]
 \centering
\includegraphics[width=0.7\linewidth]{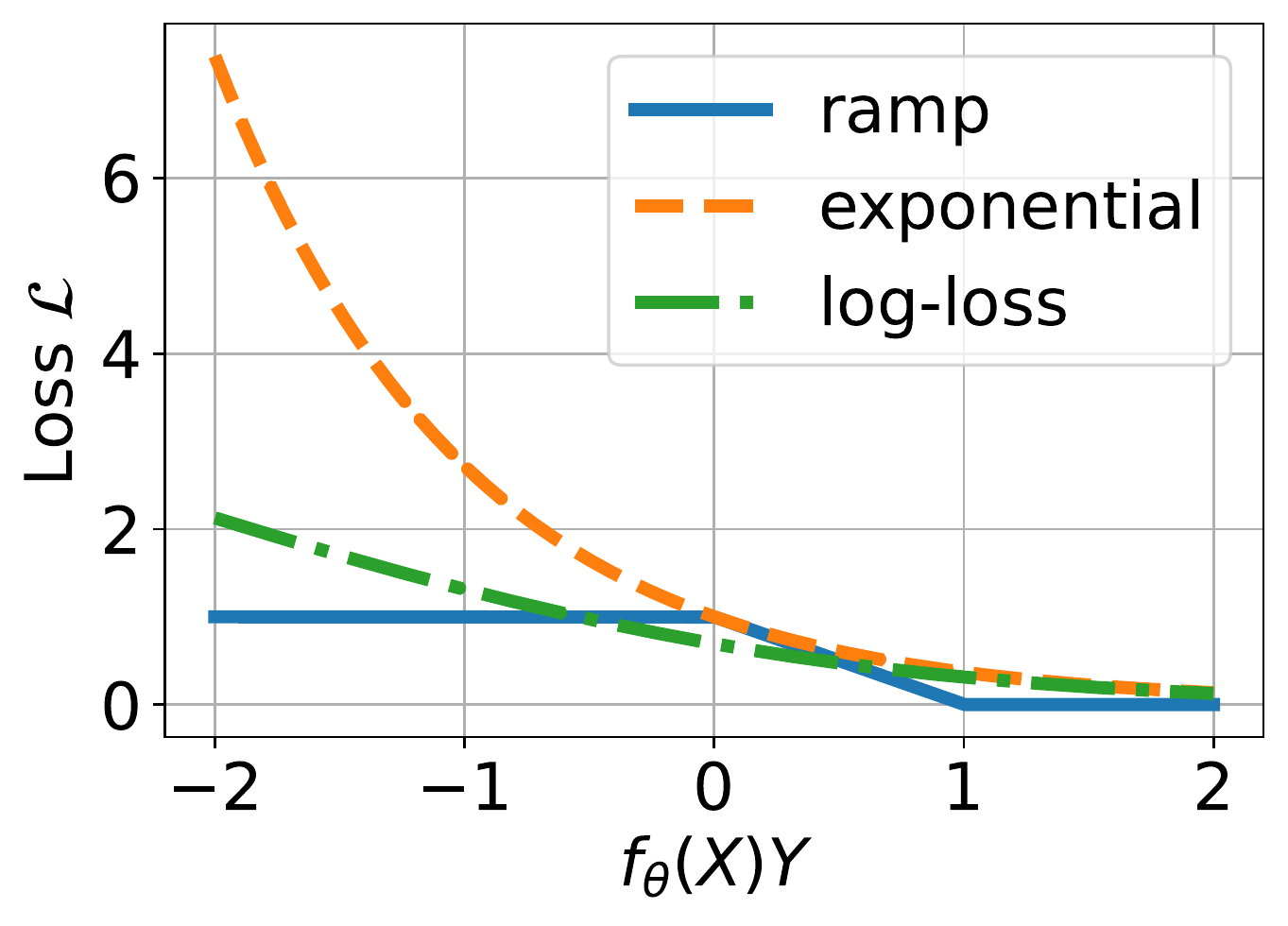}
\caption{Illustration of different loss functions}.
\label{fig:ill_loss}
\end{figure}

Notice that, as illustrated in Figure \ref{fig:ill_loss}, unbounded losses (such as log-loss or exponential loss) amplify the classification errors of misclassified samples $X$, in a way proportionately to the distance of $X$ from the decision boundary. The misclassification of a sample is captured by the expression  $f_{\theta}(X)Y<0$, while the distance to the decision boundary by expression $|f_{\theta}(X) Y|$ (x-axis). 
Notably these losses are unbounded. 

Now notice that a consequence of training a fair classifier is to push its decision boundary to dense region of the advantaged group. This is because such classifier attempts at aligning the groups classification losses, i.e., $\mathbb{E}[ \ell(f_{\theta}(X),Y) |Y=-1] = \mathbb{E}[ \ell(f_{\theta}(X),Y) |Y=1]$.
When the decision boundary is moved closer to the input samples, the classifier will inevitably become less robust to small perturbations of adversarial noise. 

On the contrary, using a bounded loss function, such as ramp loss, during fair learning, can greatly reduce the impact produced by such (outlier) samples. 

\paragraph{Effectiveness of the mitigation solution.} 
Next, this section provides additional experiments to demonstrate the effectiveness of the proposed solution to find a good tradeoff between fairness and robustness. The generality of the proposed solution is demonstrated across several architectures  (VGG-13 and  ResNet 50) and adversarial attacks ($\ell_{\infty}$ RFGSM and $\ell_2$ PGD attacks under different level of attacks $\epsilon$).

In summary, the proposed mitigation solution---which uses a bounded loss--- result in classifiers that, in the vast majority of the cases,  are fairer and more robust that those produced by models using a standard (cross entropy) loss.

\paragraph{VGG-13 and $\ell_\infty$ RFGSM attacks.}
Figures \ref{fig:bnd_err_cifar_vgg} and \ref{fig:bnd_err_utk_vgg} 
report the boundary errors attained using RFGSM attacks on fair ($\lambda>0$) and regular ($\lambda=0$) classifiers on CIFAR and UTK datasets, respectively and using a VGG 13. 
The plots compare models obtained using a cross entropy loss (top plots) and those using a Ramp loss (bottom plots). 

\begin{figure*}[t]
\centering
\includegraphics[width=0.9\linewidth]{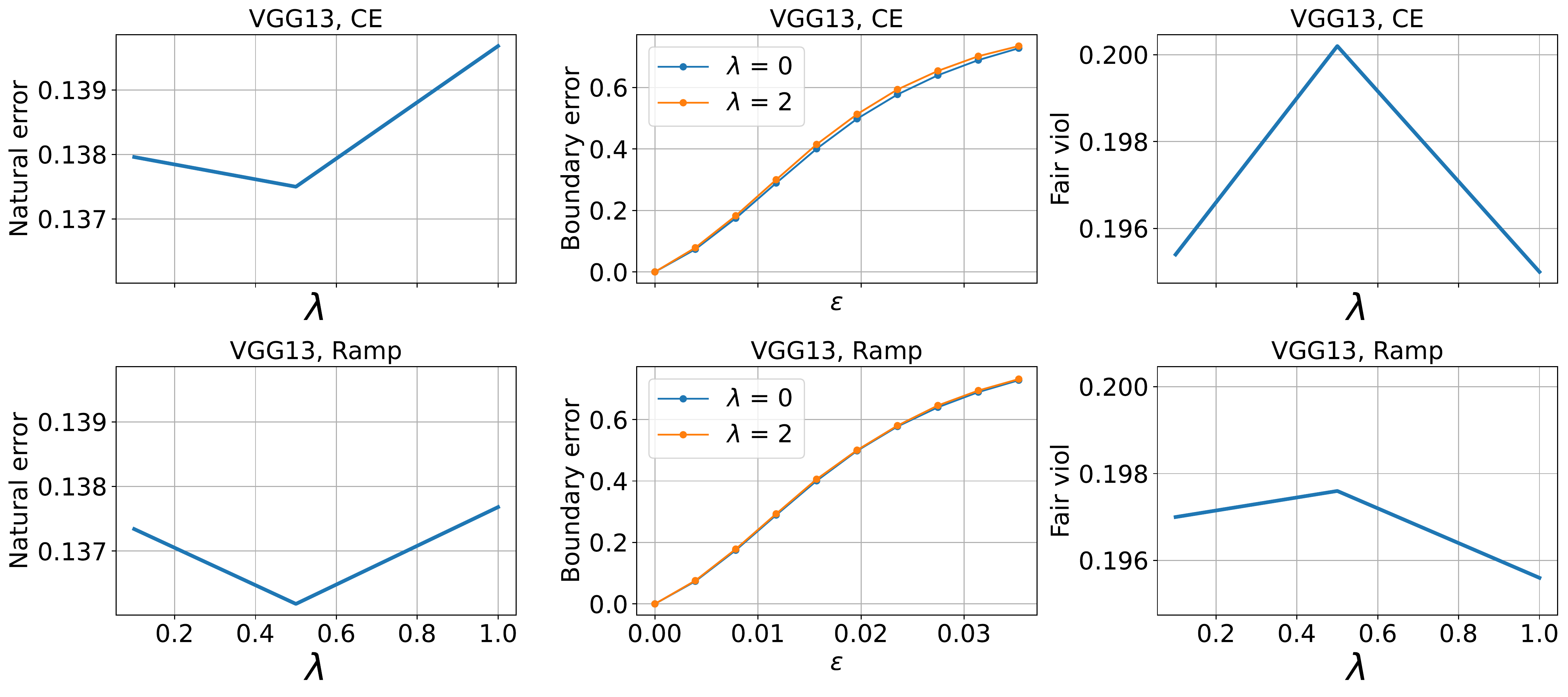}
\caption{{\bf Top}: Natural errors (left) and fairness violations (right) on the CIFAR-10 {\sl ethnicity} task at varying of the fairness parameters $\lambda$. The middle plots compares the robustness of fair ($\lambda > 0)$ vs.~natural ($\lambda=0$) classifiers to different RFGSM attack levels. 
{\bf Bottom}: Mitigating solution using the bounded Ramp loss. The base classifiers are VGG-13.}
\label{fig:bnd_err_cifar_vgg}
\end{figure*}

\begin{figure*}[t]
\centering
\includegraphics[width=0.9\linewidth]{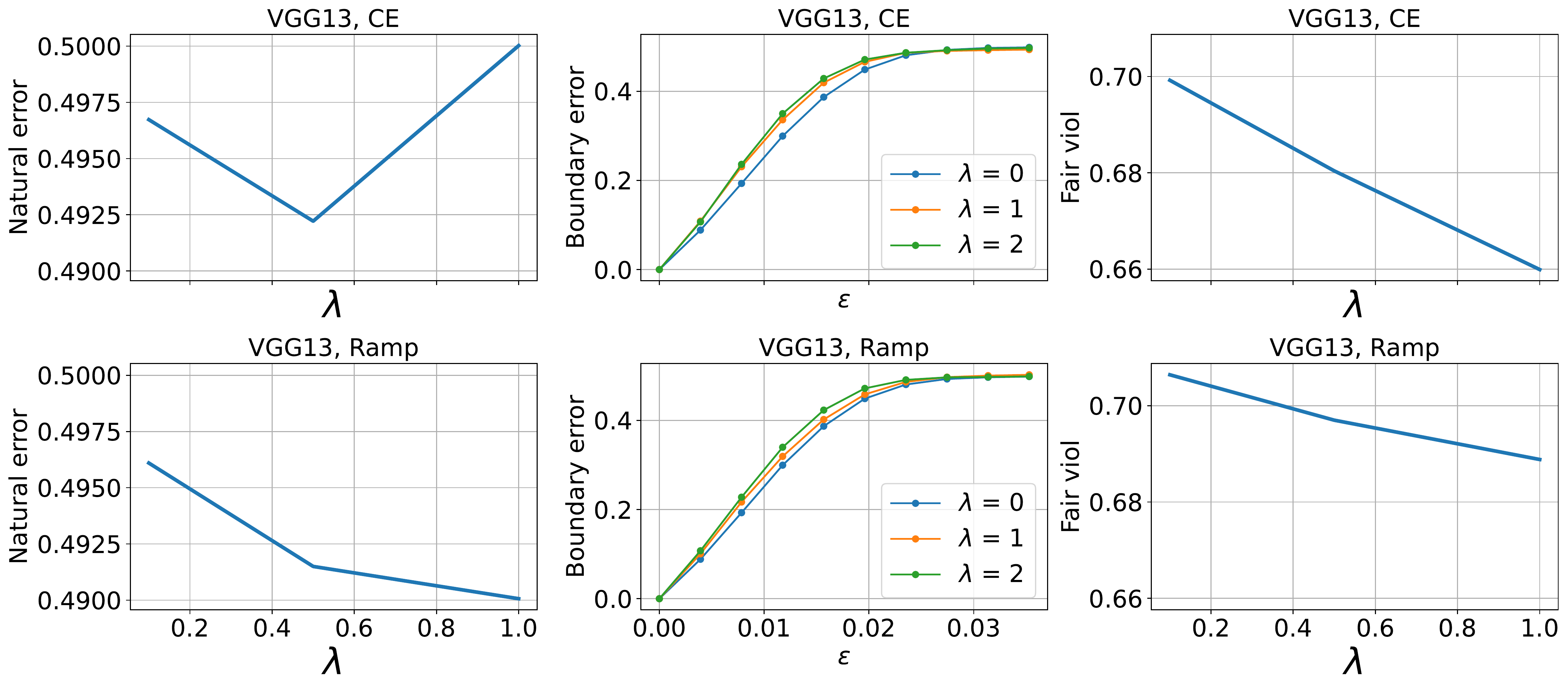}
\caption{{\bf Top}: Natural errors (left) and fairness violations (right) on the UTKFace  {\sl age bins} task at varying of the fairness parameters $\lambda$. The middle plots compares the robustness of fair ($\lambda > 0)$ vs.~natural ($\lambda=0$) classifiers to different RFGSM attack levels. 
{\bf Bottom}: Mitigating solution using the bounded Ramp loss. The base classifiers are VGG-13.}
\label{fig:bnd_err_utk_vgg}
\end{figure*}

\paragraph{VGG-13 and $\ell_2$ PGD attacks.}
Figures \ref{fig:bnd_err_utk_race_vgg_pgd} and \ref{fig:bnd_err_utk_age_vgg_pgd}.
report the boundary errors attained using a PGD attacks on fair ($\lambda>0$) and regular ($\lambda=0$) classifiers on UTK datasets and using a VGG 13. Once again, the plots compare models obtained using a cross entropy loss (top plots) and those using a Ramp loss (bottom plots).

\begin{figure*}[t]
\centering
\includegraphics[width=0.9\linewidth]{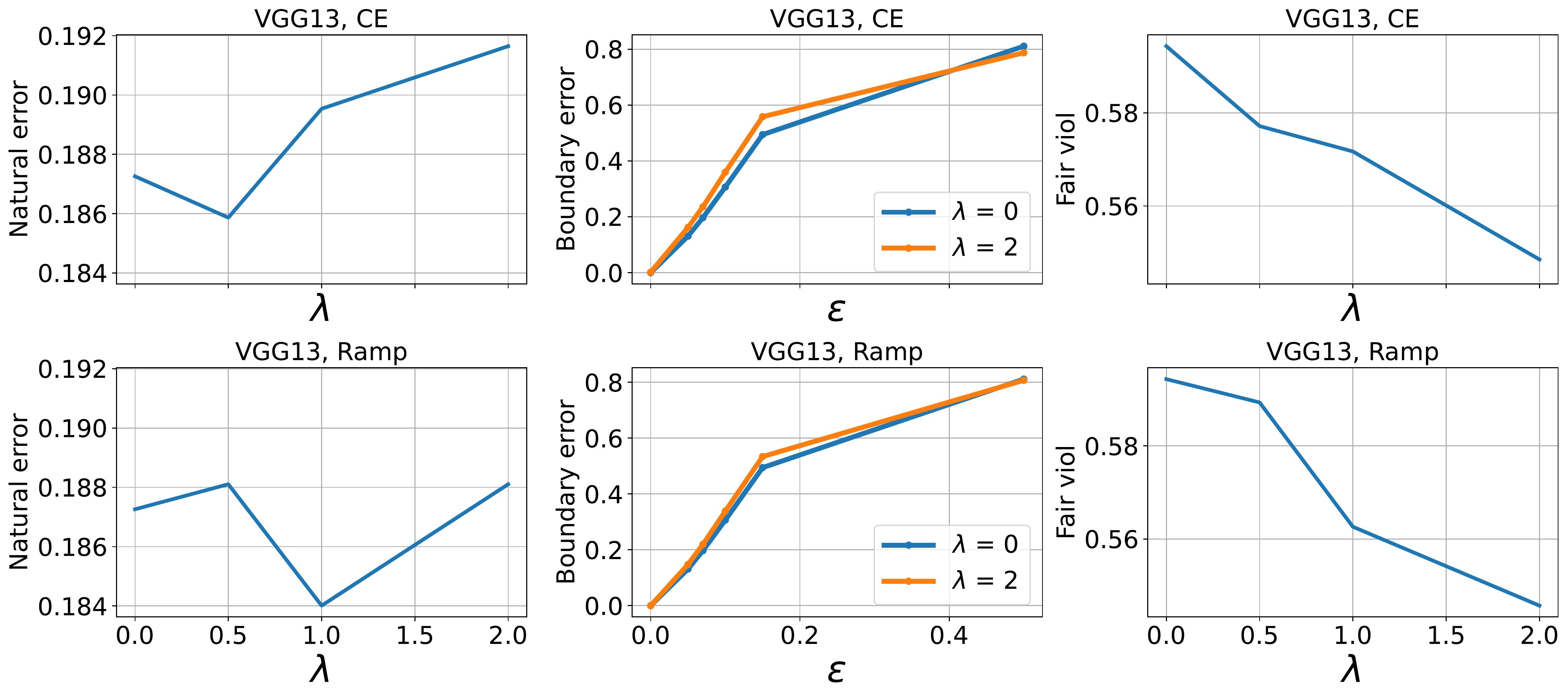}
\caption{{\bf Top}: Natural errors (left) and fairness violations (right) on the UTKFace  {\sl ethnicity} task at varying of the fairness parameters $\lambda$. The middle plots compares the robustness of fair ($\lambda > 0)$ vs.~natural ($\lambda=0$) classifiers to different $l_2$ PGD attack levels. 
{\bf Bottom}: Mitigating solution using the bounded Ramp loss. The base classifier are VGG-13.}
\label{fig:bnd_err_utk_race_vgg_pgd}
\end{figure*}

\begin{figure*}[t]
\centering
\includegraphics[width=0.9\linewidth]{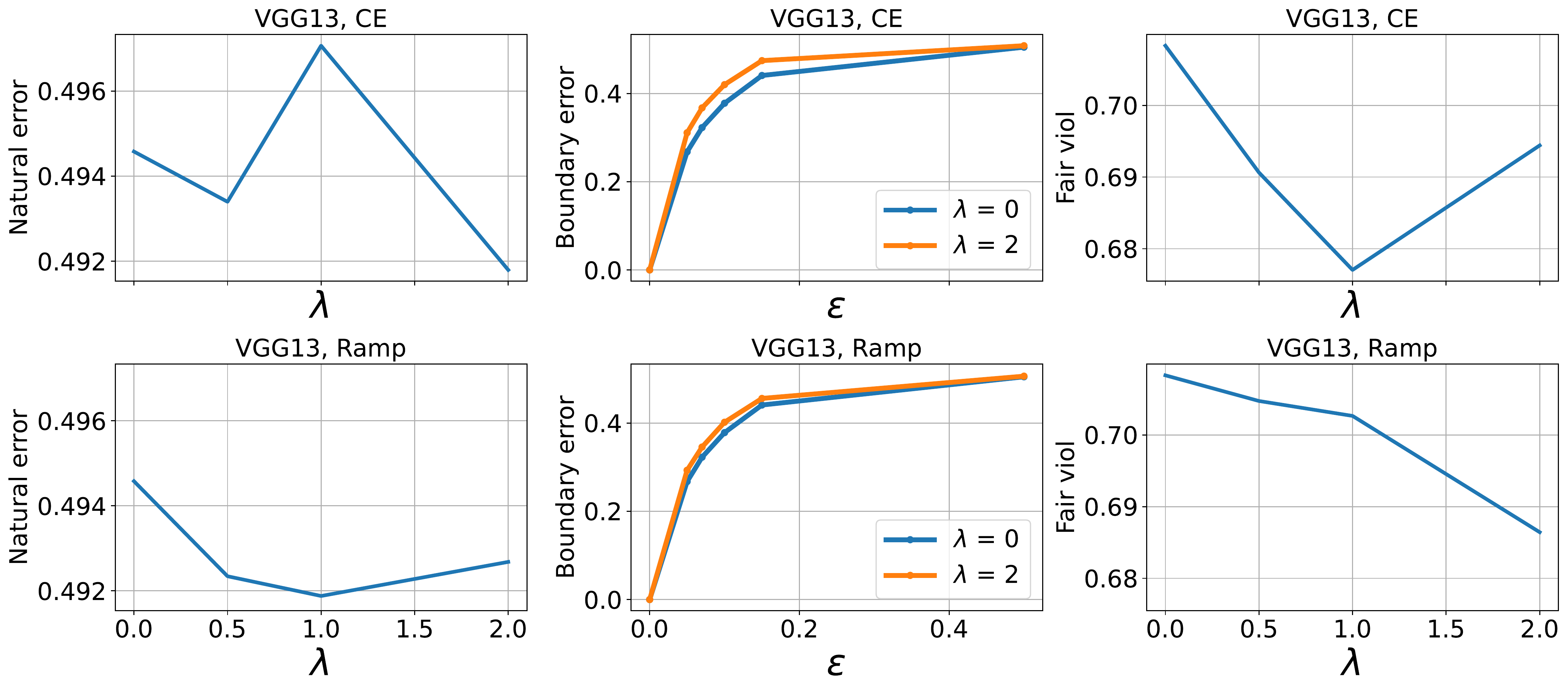}
\caption{{\bf Top}: Natural errors (left) and fairness violations (right) on the UTKFace  {\sl age bins} task at varying of the fairness parameters $\lambda$. The middle plots compares the robustness of fair ($\lambda > 0)$ vs.~natural ($\lambda=0$) classifiers to different $l_2$ PGD attack levels. 
{\bf Bottom}: Mitigating solution using the bounded Ramp loss. The base classifiers are VGG-13.}
\label{fig:bnd_err_utk_age_vgg_pgd}
\end{figure*}

\paragraph{ResNet 50 and $\ell_\infty$ RFGSM attacks}
Figures \ref{fig:bnd_err_cifar_resnet_rfgsm} and  \ref{fig:bnd_err_fmnist_resnet_rfgsm}.
report the boundary errors attained using a RFGSM attacks on fair ($\lambda>0$) and regular ($\lambda=0$) classifiers on UTK Face and FMNIST datasets, respectively and using a ResNet50. Once again, the plots compare models obtained using a cross entropy loss (top plots) and those using a Ramp loss (bottom plots).

\begin{figure*}[t]
\centering
\includegraphics[width=0.9\linewidth]{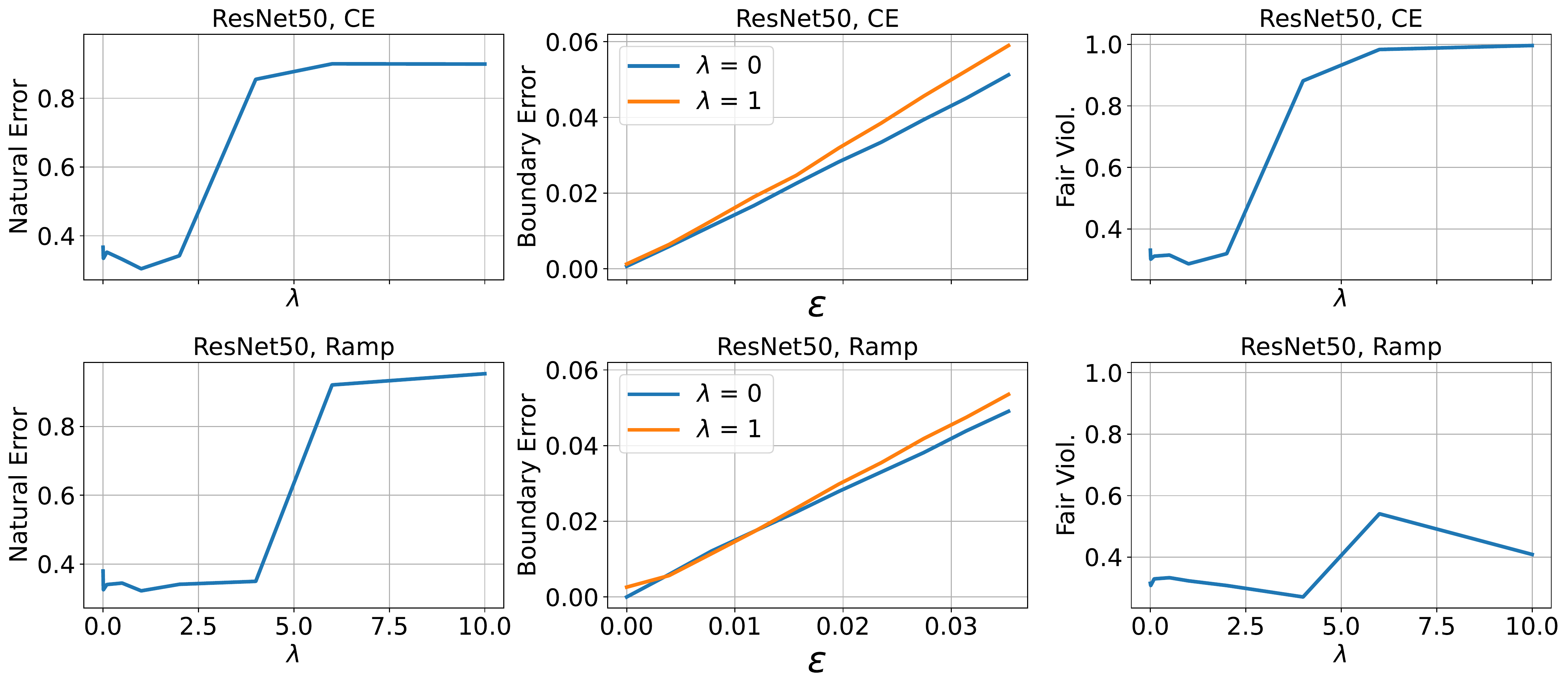}
\caption{{\bf Top}: Natural errors (left) and fairness violations (right) on the UTKFace  {\sl ethnicity} task at varying of the fairness parameters $\lambda$. The middle plots compares the robustness of fair ($\lambda > 0)$ vs.~natural ($\lambda=0$) classifiers to different $l_{\infty}$ RFGSM attack levels. 
{\bf Bottom}: Mitigating solution using the bounded Ramp loss. The base classifier are  Res Net 50.}
\label{fig:bnd_err_cifar_resnet_rfgsm}
\end{figure*}

\begin{figure*}[t]
\centering
\includegraphics[width=0.9\linewidth]{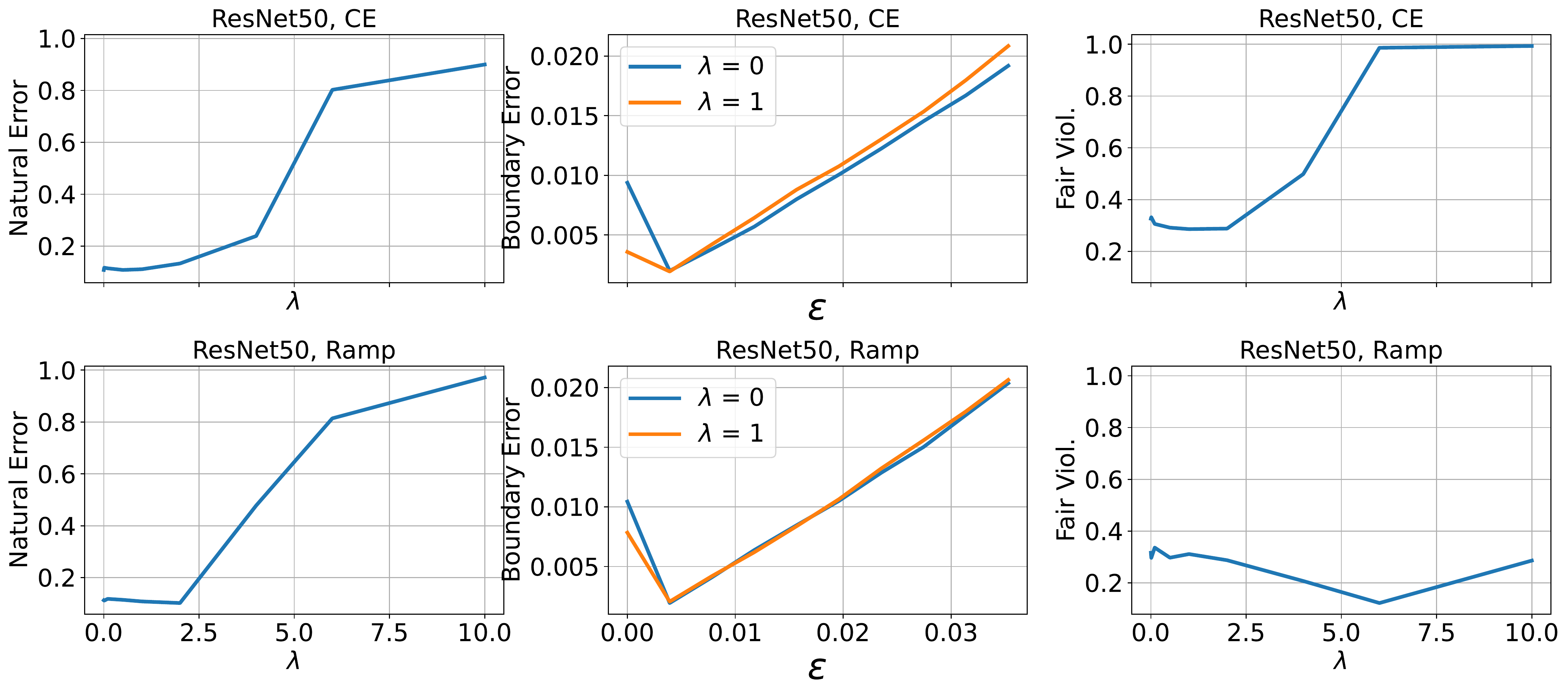}
\caption{{\bf Top}: Natural errors (left) and fairness violations (right) on the FMNIST task at varying of the fairness parameters $\lambda$. The middle plots compares the robustness of fair ($\lambda > 0)$ vs.~natural ($\lambda=0$) classifiers to different $l_{\infty}$ RFGSM attack levels. 
{\bf Bottom}: Mitigating solution using the bounded Ramp loss. The base classifier are Res Net 50.}
\label{fig:bnd_err_fmnist_resnet_rfgsm}
\end{figure*}

\paragraph{ResNet 50 and $\ell_2$  PGD attacks}
Figures \ref{fig:bnd_err_cifar_resnet_pgd} and \ref{fig:bnd_err_fmnist_resnet_pgd} 
report the boundary errors attained using a PGD attack on fair ($\lambda>0$) and regular ($\lambda=0$) classifiers on CIFAR10 and FMNIST datasets, respectively, and using a ResNet50. Once again, the plots compare models obtained using a cross entropy loss (top plots) and those using a Ramp loss (bottom plots).

% modify here 
\begin{figure*}[t]
\centering
\includegraphics[width=0.9\linewidth]{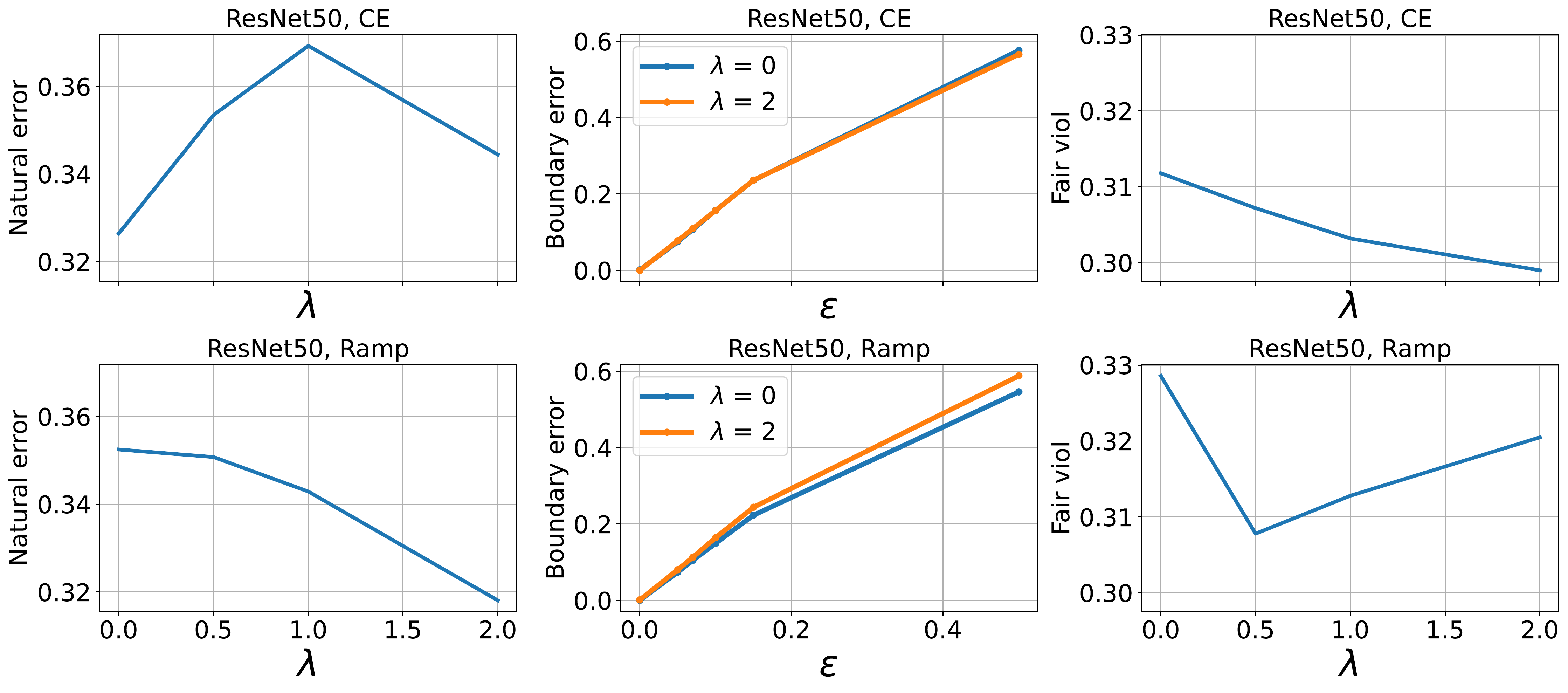}
\caption{{\bf Top}: Natural errors (left) and fairness violations (right) on the CIFAR 10 task at varying of the fairness parameters $\lambda$. The middle plots compares the robustness of fair ($\lambda > 0)$ vs.~natural ($\lambda=0$) classifiers to different $l_2$ PGD attack levels. 
{\bf Bottom}: Mitigating solution using the bounded Ramp loss. The base classifiers are ResNet 50.}
\label{fig:bnd_err_cifar_resnet_pgd}
\end{figure*}

\begin{figure*}[t]
\centering
\includegraphics[width=0.9\linewidth]{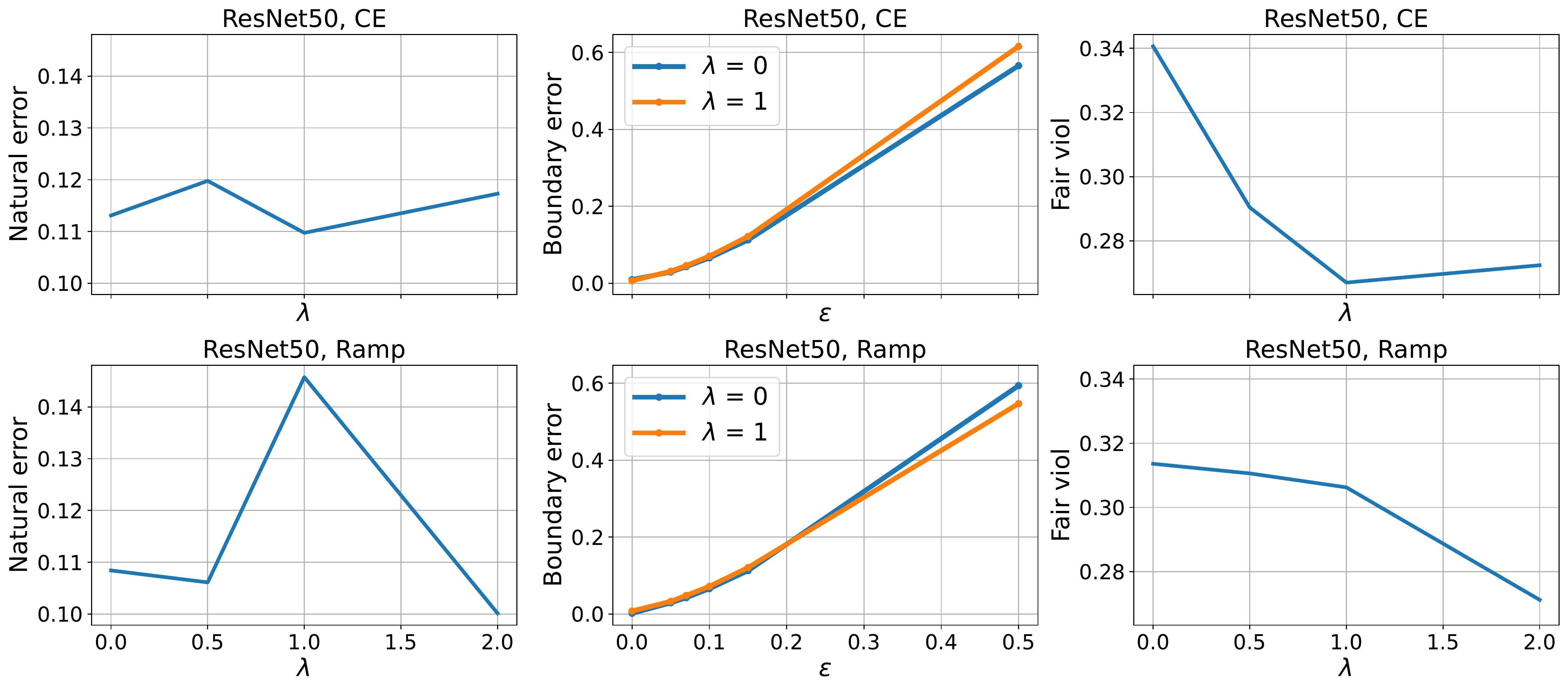}
\caption{{\bf Top}: Natural errors (left) and fairness violations (right) on the FMNIST  task at varying of the fairness parameters $\lambda$. The middle plots compares the robustness of fair ($\lambda > 0)$ vs.~natural ($\lambda=0$) classifiers to different $l_2$ PGD attack levels. 
{\bf Bottom}: Mitigating solution using the bounded Ramp loss. The base classifiers are ResNet 50.}
\label{fig:bnd_err_fmnist_resnet_pgd}
\end{figure*}

\end{document}